\documentclass[11pt]{article}
\usepackage[letterpaper]{geometry}

\usepackage[utf8]{inputenc} 
\usepackage[T1]{fontenc}    
\usepackage[parfill]{parskip}
\usepackage{amsmath,amsthm,amssymb,bm,bbm}
\usepackage{mathtools}
\usepackage{cases}

\usepackage{microtype}

\usepackage{algorithm,algorithmic}
\usepackage{color}
\usepackage{appendix}
\usepackage{lmodern}

\usepackage{url}
\usepackage[authoryear]{natbib}
\usepackage[colorlinks,citecolor=blue,urlcolor=blue,linkcolor=blue,linktocpage=true]{hyperref}

\usepackage{etoolbox}

\usepackage{cleveref}
\crefformat{equation}{(#2#1#3)}
\crefrangeformat{equation}{(#3#1#4) to~(#5#2#6)}
\crefname{equation}{}{}
\Crefname{equation}{}{}

\crefname{definition}{\textbf{definition}}{definitions}
\Crefname{definition}{Definition}{Definitions}
\crefname{assumption}{\textbf{assumption}}{assumptions}
\Crefname{assumption}{Assumption}{Assumptions}

\definecolor{maroon}{RGB}{192,80,77}

\newtheorem{theorem}{Theorem}
\newtheorem{lemma}[theorem]{Lemma}

\newtheorem{corollary}[theorem]{Corollary}
\newtheorem{definition}[theorem]{Definition}
\newtheorem{example}[theorem]{Example}
\newtheorem{assumption}{Assumption}

\newtheorem{remark}[theorem]{Remark}

\newcommand{\argmax}{\mathop{\mathrm{argmax}}}

\def\E{\mathbb{E}}
\def\P{\mathbb{P}}

\def\diag{\mathrm{diag}}

\def\R{\mathbb{R}}
\def\cA{\mathcal{A}}

\def\cD{\mathcal{D}}

\usepackage{etoolbox}
\usepackage{comment}
\usepackage{caption}
\usepackage{subcaption}
\usepackage{isomath}
\usepackage{indentfirst}
\usepackage{enumerate}
\usepackage{enumitem}
\setlist{leftmargin=10mm}
\usepackage{mathrsfs}
\usepackage{multirow}
\usepackage{amstext}
\usepackage{dsfont}
\def\ind{\mathds{1}}
\def\vind{\mathds{1}}
\def\vr{\mathbf{r}}
\def\vv{\mathbf{v}}
\def\ep{\zeta}  
\newbool{twocol}
\setbool{twocol}{false}
\usepackage{graphicx}
\usepackage{diagbox}

\newbool{compact}  
\setbool{compact}{true}
\usepackage{booktabs}       
\usepackage{amsfonts}       
\usepackage{nicefrac}       
\usepackage{xcolor}         

\usepackage{authblk}

\usepackage{array}
\newcolumntype{L}{>{\centering\arraybackslash}m{3cm}}
\usepackage{makecell}

\title{Doubly Fair Dynamic Pricing}

\author{Jianyu Xu }
\author{Dan Qiao}
\author{Yu-Xiang Wang }
\affil{Department of Computer Science\\
	University of California, Santa Barbara\\
	\texttt{\{xu\_jy15, danqiao, yuxiangw\}@ucsb.edu} }

\begin{document}

\maketitle

\begin{abstract}
We study the problem of online dynamic pricing with two types of fairness constraints: a \emph{procedural fairness} which requires the \emph{proposed} prices to be equal in expectation among different groups, and a \emph{substantive fairness} which requires the \emph{accepted} prices to be equal in expectation among different groups.  A policy that is simultaneously procedural and substantive fair is referred to as \emph{doubly fair}.  We show that a doubly fair policy must be random to have higher revenue than the best trivial policy that assigns the same price to different groups. In a two-group setting, we propose an online learning algorithm for the 2-group pricing problems that achieves $\tilde{O}(\sqrt{T})$ regret, zero procedural unfairness and $\tilde{O}(\sqrt{T})$ substantive unfairness over $T$ rounds of learning. We also prove two lower bounds showing that these results on regret and unfairness are both information-theoretically optimal up to iterated logarithmic factors. To the best of our knowledge, this is the first dynamic pricing algorithm that learns to price while satisfying two fairness constraints at the same time.


\end{abstract}

\section{Introduction}
\label{sec:introduction}
Pricing problems have been studied since \citet{cournot1897researches}. In a classical pricing problem setting such as \citet{kleinberg2003value, broder2012dynamic, besbes2015surprising}, the seller (referred as ``we'') sells identical products in the following scheme. 

\fbox{\parbox{0.98\textwidth}{Online pricing. For $t=1,2,\ldots, T$:
		\small
		\noindent
		\begin{enumerate}[leftmargin=*, align=left]
			\setlength{\itemsep}{0pt}
			\item The customer valuates the product as $y_t$.
			\item The seller proposes a price $v_t$ concurrently without knowing $y_t$.
			\item The customer makes a decision $\mathbf{1}_t = \mathbf{1}(v_t\leq y_t)$.
			\item The seller receives a reward (revenue) $r_t = v_t\cdot\mathbf{1}_t$.
		\end{enumerate}
		
	}
}

Here $T$ is the time horizon known to the seller in advance\footnote{Here we assume $T$ known for simplicity of notations. In fact, if $T$ is unknown, then we may apply a ``doubling epoch'' trick as \citet{javanmard2019dynamic} and the regret bounds are the same.}, and $y_t$'s are drawn from a fixed distribution independently. The goal is to approach an optimal price that maximizes the expected revenue-price function. In order to make this, we should learn gradually from the binary feedback and improve our knowledge on customers' valuation distribution (or so-called ``demands'' \citep{kleinberg2003value}).

In recent years, with the development of price discrimination and personalized pricing strategies, \emph{fairness} issues on pricing arose social and academic concerns \citep{kaufmann1991fairness, chapuis2012price, richards2016personalized, eyster2021pricing}. Customers are usually not satisfied with price discrimination, which would in turn undermine both their willing to purchase and the sellers' reputation. In the online pricing problem defined above, when we are selling identical items to customers from different \emph{groups} (e.g., divided by gender, race, age, etc.), it can be unfair if we propose a specific optimal price for each group: These optimal prices in different groups are not necessarily the same, and unfairness occurs if different customers are provided or buying the same item with different prices. Inspired by the concept of \emph{procedural and substantive unconscionability}\citep{elfin1988future}, we define a \emph{procedural unfairness} measuring the difference of \emph{proposed prices} between the two groups, and a \emph{substantive unfairness} measuring that of \emph{accepted prices} between the two groups. Given these notions, our goal is to approach the optimal pricing policy that maximizes the expected total revenue with no procedural and substantive unfairness.

{ The concept of procedural fairness has been well established in \citet{cohen2022price} as ``price fairness'', while the concept of the substantive fairness is new to this paper. In fact, both procedural and substantive fairness have significant impacts on customers' experience and social justice.
For instance, these notions help prevent the following two scenarios:
\begin{itemize}
	\item Perspective buyers who are women found that they are offered consistently higher average price than men for the same product.
	\item Women who have bought the product found that they paid a higher average price than men who have bought the product.
\end{itemize}
Therefore, a good pricing strategy has to satisfy both procedural and substantive fairness.
}

However, these constraints are very hard to satisfy even with full knowledge on customers' demands. If we want to fulfill those two sorts of fairness perfectly by proposing deterministic prices for different groups, the only thing we can do is to trivially set the same price in all groups and to maximize the weighted average revenue function by adjusting this uniformly fixed price with existing methods such as \citet{kleinberg2003value}. Consider the following example:
	
\begin{example}\rm
	\label{example:random_policy}
	Customers form two disjoint groups, where 30\% customers are in Group 1 and the rest 70\% are in Group 2.
	\begin{figure*}[!htbp]
		\centering
	\begin{minipage}[h]{0.50\linewidth}
		 For each price in $\{\$0.625, \$0.7, \$1\}$, customers in two groups have different acceptance rates:
		\bigskip
		
			\begin{tabular}{|l|l|l|l|}
				\hline
				\rule{0pt}{2pt}
				Acceptance Rate & \$0.625   & \$0.7     & \$1       \\ \hline
				$G_1$ (30\%)    & $3/5$ & $1/2$ & $1/2$ \\ \hline
				$G_2$ (70\%)    & $4/5$ & $4/5$ & $1/2$ \\ \hline
			\end{tabular}
		\bigskip
		
		The right figure shows the expected revenue functions of prices in each group, where the red dashed line is their weighted average by population.
	\end{minipage}
	\begin{minipage}[h]{0.49\linewidth}
		\centering
		\includegraphics[width=\textwidth]{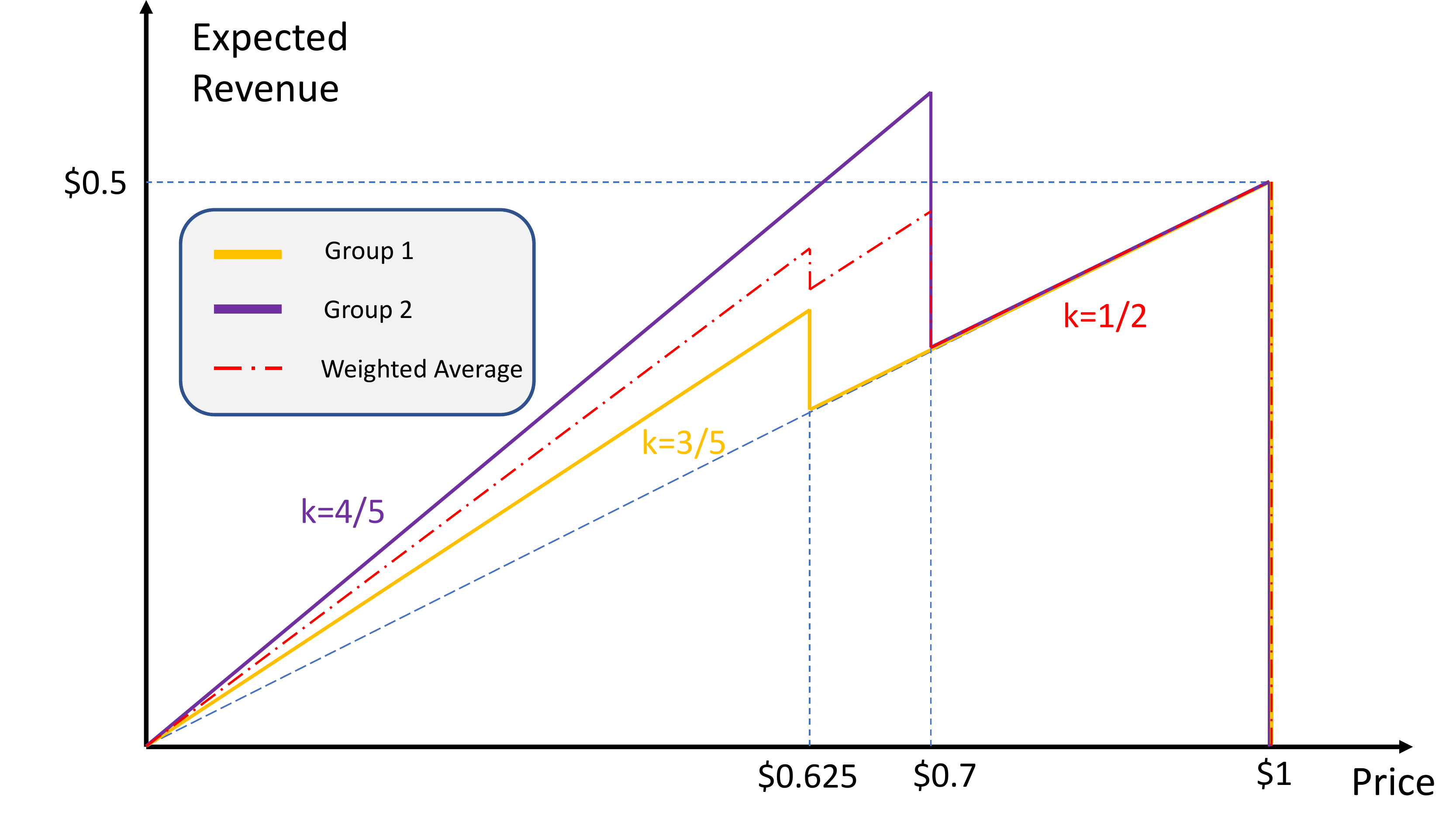}
	\end{minipage}	
	\end{figure*}
\end{example}
In \Cref{example:random_policy}, the only way to guarantee both fairness constraints is to propose the same price for both groups, and the optimal price is \textit{\$1} whose expected revenue is $\mathit{\$0.5}$ as is shown in the figure.  

However, if we instead propose a \emph{random price distribution} to each group and inspect those fairness notions \emph{in expectation}, then there may exist price distributions that satisfy both fairness constraints and achieve higher expected revenue than any fixed-price strategy. Here a price distribution is the distribution over the prices for customers, and the exact price for each customer is \emph{sampled} from this distribution \emph{independently}. This random price sampling process can be implemented by marketing campaigns such as random discounts or randomly-distributed coupons. Again, we consider \Cref{example:random_policy} and the following random policy:

\begin{itemize}
	\item For customers from $G_1$, propose \textit{\$0.625} with probability $\frac{20}{29}$ and \textit{\$1} with probability $\frac{9}{29}$.
	\item For customers from $G_2$, propose \textit{\$0.7} with probability $\frac{25}{29}$ and \textit{\$1} with probability $\frac{4}{29}$.
\end{itemize}
Under this policy, the expected proposed price and the expected accepted price in both groups are $\mathit{\$\frac{43}{58}}$ and $\mathit{\$\frac8{11}}$ respectively. Furthermore, the expected revenue is $\mathit{\$\frac{74}{145}}>\mathit{\$0.5}$, which means that this random policy performs better than the best fixed-price policy. It is worth mentioning that this is exactly the optimal random policy in this specific setting, but the proof of its optimality is highly non-trivial (and we put it in \Cref{appendix:proof_unfairness_lower_bound} as part of the proof of \Cref{theorem:unfairness_lower_bound} ).


In this work, we consider a two-group setting and we denote a \emph{policy} as the tuple of two price distributions over the two groups respectively. Therefore, we can formally define the optimal policy as follows:

\begin{equation}
\label{equ:optimal_policy}
\begin{aligned}
\pi_* =& \argmax_{\pi=(\pi^1, \pi^2)}q\cdot\mathop{\E}\limits_{v_t^1\sim\pi^1, y_t^1\sim\mathbb{D}^1}[v_t^1\cdot\ind(v_t^1\leq y_t^1)] +(1-q)\cdot\mathop{\E}\limits_{v_t^2\sim\pi^2, y_t^2\sim\mathbb{D}^2}[v_t^2\cdot\ind(v_t^2\leq y_t^2)]\\
s.t.&\quad \E_{\pi^1}[v_t^1] = \E_{\pi^2}[v_t^2]\\
&\quad \E_{\pi^1, \mathbb{D}^1}[v_t^1|\ind(v_t^1\leq y_t^1)=1] = \E_{\pi^2, \mathbb{D}^2}[v_t^2|\ind(v_t^2\leq y_t^2)=1]\\
\end{aligned}
\end{equation} 

Here $\pi^1, v_t^1, y_t^1, \mathbb{D}^1$ and $\pi^2, v_t^2, y_t^2, \mathbb{D}^2$ are the proposed price distributions, proposed prices, customer's valuations and valuation distributions of Group 1 and Group 2 respectively, and $q$ is the proportion of Group 1 among all customers. From \Cref{equ:optimal_policy}, the optimal policy under the in-expectation fairness constraints should be random in general\footnote{Notice that a fixed-price policy can also be considered as ``random''.}. However, even we know the exact $\mathbb{D}^1$ and $\mathbb{D}^2$, it is still a very hard problem to get $\pi_*$: Both sides of the second constraint in \Cref{equ:optimal_policy} are conditional expectations (i.e., fractions of expected revenue over expected acceptance rate) and is thus not convex ( and also not quasiconvex). To make it harder, the seller actually has no direct access to customers' demands $\mathbb{D}^1$ and $\mathbb{D}^2$ at the beginning. Therefore, in this work we consider a $T$-round \emph{online} learning and pricing setting, where we could learn these demands from those \emph{Boolean-censored} feedback (i.e., customers' decisions) and improve our pricing policy to approach $\pi_*$ in \Cref{equ:optimal_policy}.

In order to measure the performance of a specific policy, we define a \emph{regret} metric that equals the expected revenue difference between this policy and the optimal policy. We also quantify the procedural and substantive unfairness that equals the absolute difference of expected proposed/accepted prices in two groups. We will establish a more detailed problem setting in \Cref{sec:problem_setup}.  

\paragraph{Summary of Results} Our contributions are threefold:
\label{subsec:main_results}

\begin{itemize}
	\item We design an algorithm, FPA, that achieves an $O(\sqrt{T}d^{\frac32}\log{\frac{d\log T}{\epsilon}})$ cumulative regret with $0$ procedural unfairness and $O(\sqrt{T}d^{\frac32}\log{\frac{d\log T}{\epsilon}})$ substantive unfairness, with probability at least $(1-\epsilon)$. Here $d$ is the total number of prices allowed to be chosen from. These results indicate that our FPA is asymptotically no-regret and fair as $T$ gets large.  
	\item We show that the regret of FPA is optimal with respect to $T$, as it matches $\Omega(\sqrt{T})$ regret lower bound up to $\log\log T$ factors.
	\item We show that the unfairness of FPA is also optimal with respect to $T$ up to $\log\log T$ factors, as it has no procedural unfairness and its substantive unfairness matches the $\Omega(\sqrt{T})$ lower bound for any algorithm achieving an optimal $O(\sqrt{T})$ regret.
\end{itemize}

To the best of our knowledge, we are the first to study a pricing problem with multiple fairness constraints, where the optimal pricing policy is necessary to be random. We also develop an algorithm that is able to approach the best random pricing policy with high probability and at the least cost of revenue and fairness. 

\paragraph{Technical Novelty.} 
Our algorithm is a ``conservative policy-elimination''-based strategy that runs in epochs with doubling batch sizes as in \citet{auer2002finite}. We cannot directly apply the action-elimination algorithm for multi-armed bandits as in \citet{cesa2013online}, because the policy space is an infinite set and we cannot afford to try each one out. The fairness constraints further complicate things. Our solution is to work out just a few representative policies that are ``good-and-exploratory'', which can be used to evaluate the revenue and fairness of all other policies, then eliminate those that are unfair or have suboptimal revenue. Since we do not have direct access to the demand function, the estimated fairness constraints are changing over epochs due to estimation error, it is non-trivial to keep the target optimal policy inside our ``good policy set'' during iterations. We settle this issue by setting the criteria of a ``good policy'' conservatively.

Our lower bound is new too and it involves techniques that could be of independent interest to the machine learning theory community. Notice that it is possible to have a perfectly fair algorithm by trivially proposing the same fixed price for both groups. It is highly non-trivial to show the unfairness lower bound within the family of regret-optimal algorithms. We present our result in \Cref{subsec:unfairness_lower_bound} by establishing two similar problem settings that any algorithm cannot distinguish them efficiently and showing that a mismatch would cause a compatible amount of regret and substantive unfairness.


\section{Related Works}
\label{sec:related_works}
Here we discuss some literature that are closely related to this work. For a broader discussion, please refer to \Cref{appendix:more_related_works}.

\paragraph{Dynamic Pricing} Single product dynamic pricing problem has been well-studied through \citet{kleinberg2003value, besbes2009dynamic, wang2014close, chen2019nonparametric, wang2021multimodal}. The crux is to learn and approach the optimal of a revenue curve from Boolean-censored feedback. In specific, \citet{kleinberg2003value}  proves $\Theta(\log\log{T}), \Theta(\sqrt{T})$ and $\Theta(T^{\frac23})$ minimax regret bounds under noise-free, infinitely smooth and stochastic/adversarial valuation assumptions, sequentially. \citet{wang2021multimodal} further shows a $\Theta(T^{\frac{K+1}{2K+1}})$ minimax regret bound for $K^{\text{th}}$-smooth revenue functions. In all these works, the decision space is continuous. In our problem setting, we fix the proposed prices to be chosen in a fixed set of $d$ prices, and show a bandit-style $\Omega(\sqrt{dT})$ regret lower bound with a similar method to \citet{auer2002nonstochastic}.

\paragraph{Fairness in Machine Learning} Fairness is a long-existing topic that has been extensively studied. In the machine learning community, fairness is defined from mainly two perspectives: the \emph{group fairness} and the \emph{individual fairness}. In a classification problem, for instance, \citep{dwork2012fairness} defines these two notions as follows: (1) A group fairness requires different groups to have identical result distributions in statistics, which further includes the concepts of ``demographic parity'' (predictions independent to group attributes) and ``equalized odds'' (predictions independent to group attributes \emph{conditioning on} the true labels). In \citet{agarwal2018reductions}, these group fairness are reduced to linear constraints. The two fairness definitions we make in this work, the \emph{procedural fairness} and the \emph{substantive fairness}, belong to group fairness. (2) An individual fairness \citep{hardt2016equality} requires the difference of predictions on two individuals to be upper bounded by a distance metric of their intrinsic features. The notion ``time fairness'' is often considered as individual fairness as well. { We provide a more detailed discussion on the line of work that address fairness concerns or stochastic constraints with online learning techniques in \Cref{appendix:more_related_works}.}

\paragraph{Fairness in Pricing} Recently there are many works contributing to pricing fairness problems \citep{kaufmann1991fairness, frey1993fairness, chapuis2012price, richards2016personalized, priester2020special, eyster2021pricing,yang2022fairness}. As is stated in \citet{cohen2022price}, in a pricing problem with fairness concerns, the concept of fairness in existing works is modeled either as a utility or budget that trades-off the revenue or as a hard constraint that prevent us from taking the best action directly.\citet{cohen2022price} chooses the second model and defines four different types of fairness in pricing: price fairness, demand fairness, surplus fairness and no-purchase valuation fairness, each of which indicates the difference of prices, the acceptance rate, the surplus (i.e., (valuation $-$ price) if bought and 0 otherwise) and the average valuation of not-purchasing customers in two groups is bounded, sequentially. They show that it is impossible to achieve any pair of different fairness notions simultaneously (with deterministic prices). In fact, this can be satisfied if they allow random pricing policies. \citet{maestre2018reinforcement} indeed builds their fairness definition upon random prices by introducing a ``Jain's Index'', which indicates the homogeneity of price distributions among different groups (i.e., our procedural fairness notion). They develop a reinforcement-learning-based algorithm to provide homogeneous prices, with no theoretic guarantees.  

\citet{cohen2021dynamic} and \citet{chen2021fairness} study the online-learning-fashion pricing problem as we do.  \citet{cohen2021dynamic} considers both group (price) fairness and individual (time) fairness, and their algorithm FaPU solves this problem with sublinear regret while guaranteeing fairness. They further study the pricing problem with demand fairness that are unknown and needs learning. In this setting, they propose another FaPD algorithm that achieves the optimal $\tilde O(\sqrt{T})$ regret and guarantees the demand fairness ``almost surely'', i.e., upper bounded by $\delta\cdot T$ as a budget. \citet{chen2021fairness} considers two different sorts of fairness constraints: (1) Price fairness constraints (as in \citet{cohen2022price}) are enforced; (2)Price fairness constraints are generally defined (and maybe not accessible), where they adopt ``soft fairness constraints'' by adding the fairness violation to the regret with certain weights. In both cases, they achieve $\tilde{O}(T^{\frac45})$ regrets. These learning-based fairness requirements are quite similar to our problem setting, but in our setting the fairness constraints are non-convex (while theirs are linear) and are also optimized to corresponding information-theoretic lower bounds without undermining the optimal regret.  

\section{Problem Setup}
\label{sec:problem_setup}
In this section, we describe the problem setting of online pricing, introduce new fairness definitions and set the goal of our algorithm design.

\paragraph{Problem Description.}
\label{paragraph:symbols_and_notations}
We start with the online pricing process. 
The whole selling session involves customers from two groups ($G_1$ and $G_2$) and lasts for $T$ rounds. Prices are only allowed to be chosen from a \emph{known} and \emph{fixed} set of $d$ prices: $\mathbf{V} = \{v_1, v_2, \ldots, v_d\}$, where $0<v_1<v_2<\ldots<v_d\leq1$. 
Denote $\Delta^d = \{x\in\R_+^d, \|x\|_1 = 1\}$ as the probabilistic simplex. 
At each time $t=1,2,\ldots, T$, we propose a pricing policy $\pi = (\pi^1, \pi^2)$ consisting of two probabilistic distributions $\pi^1, \pi^2\in\Delta^d$ over all $d$ prices. A customer then arrives with an observable group attribution $G_e$ ($e\in\{1,2\}$), and we propose a price by sampling a $v_t^e$ from $\mathbf{V}$ according to distribution $\pi^e$. At the same time, the customer generates a valuation $y_t^e$ \emph{in secret}, where $y_t^e$ is sampled independently and identically from some fixed unknown distribution $\mathbb{D}_e$. Afterward, we observe a feedback $\mathbf{1}_t^e = \mathbf{1}(v_t^e\leq y_t^e)$ and receive a reward(revenue) $r_t^e = \mathbf{1}_t^e\cdot v_t^e$.

\paragraph{Key Quantities.}

Here we define a few quantities and functions that is necessary to formulate the problem. 
Denote $\vv:=[v_1, v_2, \ldots, v_d]^{\top}$, $[d]:=\{1,2,\ldots, d\}$ and $\vind:=[1,1,\ldots, 1]^{\top}\in\R^d$ for simplicity. Denote $F_e(i):=\Pr_{\mathbb{D}_e}[y_t^e\geq v_i], e=1,2, i\in[d]$ as the probability of price $v_i$ being accepted in $G_e$, and we know that $F_e(1)\geq F_e(2)\geq\ldots\geq F_e(d)$. Notice that all $F_e(i)$'s are \emph{unknown} to us. Define a matrix $F_e := \diag(F_e(1), F_e(2), \ldots, F_e(d))$. 

As a result, for a customer from $G_e$ ($e\in\{1,2\}$), we know that
\begin{itemize}
	\item The expected proposed price is $\vv^{\top}\pi^e$.
	\item The expected reward(revenue) is $\vv^{\top}F_e\pi^e$.
	\item The expected acceptance rate is $\vind^{\top}F_e\pi^e$.
	\item The expected accepted price is $\frac{\vv^{\top}F_e\pi^e}{\vind^{\top}F_e\pi^e}$. 
\end{itemize}
Denote the proportion of $G_1$ in all potential customers as $q$ ($0<q<1$) which is fixed and \emph{known} to us, and we assume that every customer is chosen from all potential customers uniformly at random. As a consequence, we can define the expected revenue of a policy $\pi$.
\begin{definition}[Expected Revenue] For any pricing policy $\pi=(\pi^1, \pi^2)\in\Pi$, define its expected revenue (given $F_1$ and $F_2$) as the weighted average of the expected rewards of $G_1$ and $G_2$.
	\begin{equation}
	\label{equ:expected_revenue}
	\begin{aligned}
	R(\pi; F_1, F_2):=&\Pr[\text{Customer is from } G_1]\cdot\E[r_t^1]+\Pr[\text{Customer is from } G_2]\cdot\E[r_t^2]\\
	=& q\cdot \vv^{\top}F_1\pi^1 + (1-q)\cdot \vv^{\top}F_2\pi^2
	\end{aligned}
	\end{equation}
	\label{def:expected_revenue}
\end{definition}
\vspace{-3mm}
Also, we can define the two different unfairness notions based on these results above.
\begin{definition}[Procedural Unfairness] For any pricing policy $\pi\in\Pi$, define its procedural unfairness as the absolute difference between the expected \emph{proposed} prices of two groups.
	\begin{equation}
	\label{equ:procedural_unfairness}
	U(\pi):=|\vv^{\top}\pi^1-\vv^{\top}\pi^2|= |\vv^{\top}(\pi^1-\pi^2)|.
	\end{equation}
	\label{def:procedural_unfairness}
\end{definition}
\vspace{-3mm}
Procedural unfairness is totally tractable as we have full access to $\vv^{\top}$ and $\pi$. Therefore, we can define a policy family $\Pi:=\{\pi=(\pi^1, \pi^2), U(\pi)=0\}$ that contains all policies with no procedural unfairness. Now we define a substantive unfairness as another metric. 
\begin{definition}[Substantive Unfairness]
	For any pricing policy $\pi\in\Pi$, define its substantive unfairness as the difference between the expected \emph{accepted} prices of two groups.
	\begin{equation}
	\label{equ:substantive_unfairness}
	\begin{aligned}
	S(\pi; F_1, F_2):=&\left|\E[v^1|v^1\sim\pi^1, v^1\text{ being accepted }] - \E[v^2|v^2\sim\pi^2, v^2\text{ being accepted }]\right|\\
	       =&\left|\frac{\vv^{\top}F_1\pi^1}{\vind^{\top}F_1\pi^1}-\frac{\vv^{\top}F_2\pi^2}{\vind^{\top}F_2\pi^2}\right|.
	\end{aligned}
	\end{equation}
	\label{def:substantive_unfairness}
\end{definition}
Substantive unfairness is not as tractable as procedural unfairness, as we have no direct access to the true $F_1$ and $F_2$. Ideally, the optimal policy that we would like to achieve is:
\begin{equation}
\label{equ:pi_star}
\begin{aligned}
\pi_* =& \argmax_{\pi=(\pi^1, \pi^2)\in\Pi} R(\pi; F_1, F_2)\\
s.t.&\quad U(\pi) = 0, \quad S(\pi; F_1, F_2) = 0.\\
\end{aligned}
\end{equation}

The feasibility of this problem is trivial: policies such as $\pi^1 = \pi^2 = [0, \ldots, 0,1,0,\ldots,0]^{\top}$ (i.e., proposing the same fixed price despite the customer's group attribution) are always feasible. However, this problem is in general highly non-convex and non-quasi-convex. Finally, we define a \emph{(cumulative) regret} that measure the performance of any policy $\pi$:
\begin{definition}[Regret]
	For any algorithm $\cA$, define its cumulative regret as follows:
	\begin{equation}
	\label{equ:regret}
	Reg_T(\cA):=\sum_{t=1}^T Reg(\pi_t; F_1, F_2):=\sum_{t=1}^T R(\pi_*; F_1, F_2)-R(\pi_t; F_1, F_2).
	\end{equation}
	Here $\pi_t$ is the policy proposed by $\cA$ at time $t$.
	\label{def:regret}
\end{definition}
Notice that we define the per-round regret by comparing the performance of $\pi_t$ with the optimal policy $\pi_*$ under constraints. Therefore, $Reg(\pi_t; F_1, F_2)$ is possible to be negative if $\pi\in\Pi$ but $U(\pi_t) > 0\text{ or } S(\pi_t; F_1, F_2) > 0$. Similarly, we define a \emph{cumulative substantive unfairness} as $S_T(\cA):=\sum_{t=1}^T S(\pi; F_1, F_2)$.

\paragraph{Goal of Algorithm Design} Our ultimate goal is to approach $\pi_*$ in the performance. In the online pricing problem setting we adopt, however, we cannot guarantee $S(\pi_t; F_1, F_2)=0$ for all $\pi_t$ we propose at time $t=1,2,\ldots, T$ since we do not know $F_1$ and $F_2$ in advance. Instead, we may suffer a gradually vanishing unfairness as we learn $F_1$ and $F_2$ better. Therefore, our goal in this work is to design an algorithm that guarantees an optimal regret while suffering $0$ cumulative procedural unfairness and the least cumulative substantive unfairness. 

\paragraph{Technical Assumptions.} Here we make some mild assumptions that help our analysis. \label{subsec:assumption}

\begin{assumption}[Least Probability of Acceptance]
	\label{assumption:f_min}
	There exists a fixed constant $F_{\min}>0$ such that $F_e(d)\geq F_{\min}, e=1,2$.
\end{assumption}

\Cref{assumption:f_min} not only ensures the definition of expected accepted price to be sound (by ruling out these unacceptable prices), but also implies $S(\pi, F_1, F_2)$ to be Lipschitz. Besides, we can always achieve this by reducing $v_d$. We will provide a detailed discussion in \Cref{paragraph:improvement_on_technical_assumptions}.

\begin{assumption}[Number of Possible Prices]
	\label{assumption:d_upper_bound}
	We treat $d$, the number of prices, as an amount independent from $T$. Also, we assume $d = O(T^{\frac13})$.
\end{assumption}

\Cref{assumption:d_upper_bound} is a necessary condition of applying $\Omega(\sqrt{dT})$ regret lower bound, and here we make it to show the optimality of our algorithm w.r.t. $T$. 

\section{Algorithm}
\label{sec:algorithm}
In this section, we propose our Fairly Pricing Algorithm (FPA) in Algorithm 1 and then discuss the techniques we develop and apply to achieve the ``no-regret'' and ``no-unfairness'' goal.

\begin{algorithm}[!htbp]
	\caption{Fairly Pricing Algorithm (FPA)}
	\label{algorithm}
	\begin{algorithmic}[1]
		\STATE {\bfseries Input:} {Time horizon $T$, prices set $\mathbf{V}$, error probability $\epsilon$, universal constant $L$, proportion $q$.}
		\STATE {\bfseries Before Epochs:} Keep proposing the highest price $v_d$ for $\tau_0=2\log{T}\log{\frac{16}{\epsilon}}$ rounds. Estimate the average rate of acceptance in two groups as $\bar{F}_d(1)$ and $\bar{F}_d(2)$. Take $\hat{F}_{\min} = \frac{\min\{\bar{F}_d(1), \bar{F}_d(2)\}}2$. 
		\STATE {\bfseries Initialization:} { Parameters $C_q, c_t$. Epoch length $\tau_k=O(d\sqrt{T}\cdot2^{k})$, reward uncertainty $\delta_{k,r}$ and unfairness uncertainty $\delta_{k,s}$ for $k=1,2,\ldots, O(\log{T})$. (To be specified in \Cref{appendix:proof_upper_bound} )\\
			Candidate policy set $\Pi_1 = \Pi:=\{\pi=(\pi^1, \pi^2), U(\pi)=0\}$ and price index set $I_0^1 = I_0^2 = [d]$.}
		\FOR{Epoch $k=1, 2, \ldots$} 
		\STATE Set $A_{k}=\emptyset$, $ I_{k}^1=I_{k-1}^1$ and $I_{k}^2=I_{k-1}^2$.
		\FOR{ Group $e=1,2$ and for { price index} $i\in I_{k-1}^e$,}
		\STATE\COMMENT{Pick up policy maximizing each probability:}
		\STATE Get $\tilde\pi_{k,i,e}=\argmax_{\pi\in\Pi_{k}} \pi^e(i)$.         
		\STATE If $\tilde\pi_{k,i,e}^e(i)\geq\frac1{\sqrt{T}}$, let $A_{k} = A_{k}\cup\{\tilde{\pi}_{k,i,e}\}$. Otherwise, remove $i$ from $I_{k}^e$.
		\ENDFOR
		\STATE Set $M_{k,e}(i)=N_{k,e}(i) = 0, \forall i\in[d], e=1,2$.
		\FOR{each policy $\pi\in A_k$,}
		\STATE \COMMENT{Sample random prices at $t=1,2,\ldots,\frac{\tau_k}{|A_k|}$ repeatedly:}
		\STATE Run $\pi$ for a batch of $\frac{\tau_k}{|A_k|}$ rounds. 
		\STATE For each time a price $v_i$ is \emph{proposed} in $G_e$, set $M_{k,e}(i) += 1$.
		\STATE For each time a price $v_i$ is \emph{accepted} in $G_e$, set $N_{k,e}(i) += 1$.
		\ENDFOR
		\STATE For $e=1,2$, set $\bar{F}_{k,e}(i) = \max\{\frac{N_{k,e}(i)}{M_{k,e}(i)}, \hat{F}_{\min}\}$ for $i\in I_k^e$, and $\bar{F}_{k,e}(i)=\hat{F}_{\min}$ otherwise.
		\STATE Let $\hat{F}_{k,e} = \diag(\bar{F}_{k,e}(1), \bar{F}_{k,e}(2), \ldots, \bar{F}_{k,e}(d)), e=1,2$.
		\STATE Solve the following optimization problem and get the empirical optimal policy $\hat{\pi}_{k,*}$.
		\begin{equation}
		\label{equ:empirical_optimizer}
		\begin{aligned}
		\hat{\pi}_{k,*}&=\argmax_{\pi\in\Pi_k} R(\pi, \hat{F}_{k,1}, \hat{F}_{k,2}),\ \text{s.t.}\ S(\pi, \hat{F}_{k,1}, \hat{F}_{k,2})\leq\delta_{k,s}.
		\end{aligned}
		\end{equation}
		\STATE To eliminate largely suboptimal or unfair policies, construct
		\begin{equation}
		\small
		\label{equ:pi_k+1}
		\Pi_{k+1}=\{\pi: \pi\in\Pi_k, S(\pi, \hat{F}_1, \hat{F}_2)\leq\delta_{k,s}, R(\pi, \hat{F}_1, \hat{F}_2)\geq R(\hat\pi_{k,*}, \hat{F}_1, \hat{F}_2)-\delta_{k,r}-L\cdot\delta_{k,s}\}.
		\end{equation}
		\ENDFOR
	\end{algorithmic}
\end{algorithm}

\subsection{Algorithm Components}
\label{subsec:tricks_in_design}
\Cref{algorithm} takes the following inputs: time horizon $T$, price set $\mathbf{V}$, error probability $\epsilon$, a universal constant $L$ { as the coefficient of the performance-fairness tradeoff on constraint relaxations} (see \Cref{lemma:small_relaxation_gain}), and $q$ as the proportion that $G_1$ takes. In the ``before epochs'' stage, we keep proposing the highest price $v_d$ for $\tau_0=O(\log T)$ rounds to estimate(lower-bound) the least acceptance rate $F_{\min}$. We also adopt the following techniques that serve as components of FPA and contribute to its no-regret and no-unfairness performance.

\paragraph{Doubling Epochs} Despite the ``before epochs'' stage, we divide the whole time space into epochs $k=1,2,\ldots$, where each epoch $k$ has a length $\tau_k=O(\sqrt{T}\cdot2^k)$ that doubles that of epoch $(k-1)$. Within each epoch $k$, we run a set of ``good-and-exploratory policies'' (to be introduced in \Cref{paragraph:good_and_exploratory}) with equal shares of $\tau_k$. At the end of each epoch $k$, we update the estimates of $F_1$ and $F_2$, eliminate the sub-optimal policies and update the set of ``good-and-exploratory policies'' for the next epoch. Since the estimates of parameters get better as $k$ increases, a doubling-epoch trick would ensure that we run better policies in longer { epochs} and therefore save the regret. 

\paragraph{Policy Eliminations} At the end of each epoch $k$, we update the candidate policy set by eliminating those substantially sub-optimal policies: Firstly, we select an empirical optimal policy $\hat{\pi}_{k,*}$ that maximizes $R(\pi,\hat{F}_{k,1}, \hat{F}_{k,2})$ while guaranteeing $S(\pi,\hat{F}_{k,1}, \hat{F}_{k,2})\leq\delta_{k,s}$. After that, we eliminate those policies that satisfy one of the following two criteria:
\begin{itemize}
	\item Large unfairness: $S(\pi,\hat{F}_{k,1}, \hat{F}_{k,2})>\delta_{k,s}$, or
	\item Large regret: $R(\pi, \hat{F}_{k,1}, \hat{F}_{K,2})<R(\hat\pi_{k,*}, \hat{F}_{k,1}, \hat{F}_{k,2})-\delta_{k,r}-L\cdot\delta_{k,s}$.
\end{itemize}
Here we adopt two subtractors on the regret criteria: $\delta_{k,r}$ for the estimation error in $R(\pi)$ caused by $\hat{F}_{k,e}$, and $L\cdot\delta_{k,s}$ for the possible increase of optimal reward by allowing $S(\pi)\leq \delta_{k,s}$ instead of $S(\pi)=0$. In this way, we can always ensure the optimal policy $\pi_*$ (i.e., the solution of \Cref{equ:pi_star}) to remain and also guarantee the other remaining policies perform similarly to $\pi_*$.

\paragraph{Good-and-Exploratory Policies}
\label{paragraph:good_and_exploratory}
Although all remaining policies perform good, not all of them are suitable of running in consideration of exploration. It is important to update the estimates of all $F_1(i)$ and $F_2(i)$ as they are required in the policy elimination. We solve this issue by keeping a set of \emph{good-and-exploratory} policies: After eliminating sub-optimal policies at the end of previous epoch, for each price $v_i$ in group $G_e$ we find out a policy in the remaining policies that maximizes the probability of proposing $v_i$ in $G_e$ at the beginning of current epoch. The larger this probability is, the more times $v_i$ can be chosen in $G_e$, which would lead to a better estimate of $F_e(i)$. Here we give up to estimate the acceptance probability of those $v_i$ with $\leq\frac1{\sqrt{T}}$ to be chosen by the optimal policy $\pi_*$, as it would not affect the elimination process and the performance substantially.


\subsection{Computational Cost}
\label{subsec:property_of_algorithm}

Our FPA algorithm is \emph{oracle-efficient} due to the doubling-epoch design, as we only run each oracle and update each parameter for $O(\log T)$ times. However, the implementation of these oracles could be time-consuming: On the one hand, each set $\Pi_k$ contains infinite policies, and a discretization would lead to exponential computational cost w.r.t. $d$. On the other hand, both \Cref{equ:empirical_optimizer} and \Cref{equ:pi_k+1} are highly non-convex on the constraints and are hard to solve with off-the-shelf methods.



\section{Regret and Unfairness Analysis}
\label{sec:regret_analysis}
In this section, we analyze the regret and unfairness of our FPA algorithm. We first present an $\tilde{O}(\sqrt{T}d^{\frac32})$ regret upper bound along with an $\tilde{O}(\sqrt{T}d^{\frac32})$ unfairness upper bound. Then we show both of them are optimal (w.r.t. $T$) up to $\log\log T$ factors by presenting matching lower bounds. 

\subsection{Regret Upper Bound}
\label{subsec: regret_upper_bound}

First of all, we propose the following theorem as the main results for our \Cref{algorithm} (FPA).
\begin{theorem}[Regret and Unfairness]
	\label{theorem:regret_and_unfairness}
	FPA guarantees an $O(\sqrt{T}d^{\frac32}\log{\frac{d\log T}{\epsilon}})$ regret with no procedural unfairness and an $O(\sqrt{T}d^{\frac32}\log{\frac{d\log T}{\epsilon}})$ substantive unfairness with probability $1-\epsilon$.
\end{theorem}
\begin{proof}[Proof sketch]
	We prove this theorem by induction w.r.t. epoch index $k$. Firstly, we assume that $\pi_*\in\Pi_k$, which naturally holds as $k=1$. Meanwhile, we show a high-probability bound on the estimation error of each $F_e(i)$ for epoch $k$, according to concentration inequalities. Given this, we derive the estimation error bound of $R(\pi, F_1, F_2)$ and $S(\pi, F_1, F_2)$ for each policy $\pi\in\Pi_{k}$ in epoch $k$. After that, we bound the regret and unfairness of each policy remaining in $\Pi_{k+1}$, and therefore bound the regret and unfairness of epoch $(k+1)$ with high probability. Finally, we show that optimal fair policy $\pi_*$ (defined in \Cref{equ:pi_star}) is also in $\Pi_{k+1}$, which matches the induction assumption for Epoch $(k+1)$. By adding up these performance over epochs, we get the cumulative regret and unfairness respectively. Please refer to \Cref{appendix:proof_upper_bound} for a detailed proof. 
\end{proof}
\begin{remark}
	Our algorithm guarantees $O(\sqrt{T}\log\log T)$ regret and unfairness simultaneously, whose average-over-time match the generic estimation error of $O(\frac1{\sqrt{T}})$. It implies that these fairness constraints do not bring informational obstacles to the learning process. In fact, these upper bounds are tight up to $O(\log\log T)$ factors, which are shown in \Cref{theorem:regret_lower_bound} and \Cref{theorem:unfairness_lower_bound}.
\end{remark}
\subsection{Regret Lower Bound}
\label{subsec:regret_lower_bound}
Here we show the regret lower bound of this pricing problem.
\begin{theorem}[Regret lower bound]
	\label{theorem:regret_lower_bound}
	Assume $d\leq T^{\frac13}$. Given the online two-group fair pricing problem and the regret definition as \Cref{equ:regret}, any algorithm would at least suffer an $\Omega(\sqrt{dT})$ regret.
\end{theorem}
We may prove \Cref{theorem:regret_lower_bound} by a reduction to online pricing problem with no fairness constraints: Given a problem setting where the two groups are identical, i.e. $F_1(i)=F_2(i), \forall i\in[d]$, and let $q=0.5$. Notice that any policy satisfying $\pi^1=\pi^2$ is fair, and the optimal policy is to always propose the best fixed price. Therefore, this can be reduced to an online identical-product pricing problem, and we present a bandit-style lower bound proof in \Cref{appendix:proof_regret_lower_bound} inspired by \citet{auer2002nonstochastic}.
\subsection{Unfairness Lower Bound with Optimal Revenue}
\label{subsec:unfairness_lower_bound}
Here we show that any optimal algorithm has to suffer an $\Omega(\sqrt{T})$ substantive unfairness. 

\begin{theorem}[Substantive Unfairness Lower Bound]
	\label{theorem:unfairness_lower_bound}
	 For any constant $C_x$, there exists constants $C_u>0$ such that any algorithm with an $C_x\cdot T^{\frac12}$ cumulative regret and zero procedural unfairness has to suffer an $C_u\cdot T^{\frac12}$ substantive unfairness.  
\end{theorem}
It is worth mentioning that this result is different from ordinary lower bounds on the regret, as it also requires the algorithm to be optimal. In general, we propose 2 different problem settings, and we show the following four facts:
\begin{itemize}
	\item No algorithm can perform well in both settings.
	\item Any algorithm cannot distinguish between the two settings very efficiently.
	\item Not trying to distinguish between them would suffer either a very large regret or a very large substantive unfairness, and therefore we cannot do this very often.
	\item Having tried but failed in distinguishing between them would definitely lead to a large substantive unfairness.
\end{itemize}
In order to prove this, we make use of \Cref{example:random_policy} presented in \Cref{sec:introduction}. One of the settings is exactly \Cref{example:random_policy}, and the other one is identical to it except these $0.5$ probabilities are now $(0.5-\ep)$ in both groups. We get close-form solutions to both problem settings and show that they are indistinguishable in information theory. Please refer to \Cref{appendix:proof_unfairness_lower_bound} for more details.

\section{Discussion}
\label{sec:discussion}
Here we discuss some open issues and potential extensions of this work. For more discussions on settings, techniques and social impacts, please refer to \Cref{appendix:more_discussion}.

\paragraph{Improvement on Technical Assumptions.}\label{paragraph:improvement_on_technical_assumptions}
In this work, we assumed the existence of a lower bound $F_{\min}>0$ of the acceptance rate of all prices for both groups. This assumption is stronger than our expectation, as the seller would not know the highest price that customers would accept. We assume this for two reasons: (1) Without assuming $F_e(i)>0$, the substantive unfairness function might be undefined. For instance, if a pricing policy is completely unacceptable in $G_1$ (with \emph{no} accepted prices) but is acceptable in $G_2$, then is it a fair policy? (2) With a constantly large probability of acceptance, we can estimate every $F_e(i)$ and bound it away from $0$ and therefore leads to the Lipschitzness of $S(\pi, \hat{F}_1, \hat{F_2})$. However, there might exist an algorithm that works for $F_e(i)>0$ generally and maintains these optimalities as well, which is an open problem to the future.  

\paragraph{Feelings of Fairness in FPA.} In our FPA algorithm, notice that we run each $\tilde\pi\in A_k$ for a continuous batch of $\frac{\tau_k}{|A_k|} =\Omega(\sqrt{T}\cdot2^k)$, which is long enough for customers to experience the fairness by comparing their proposed prices and accepted prices with customers from the other group.

\paragraph{Relaxation on Substantive Fairness.}\label{paragraph:relaxation}
In this work, our algorithm approached the optimal policy as the solution of \Cref{equ:pi_star} through an online learning framework. This ensures an asymptotic fairness as $T\rightarrow+\infty$, but we still cannot guarantee an any-time fair algorithm precisely. Therefore, it is more practical to consider the following inequality-constraint optimization problem:
\begin{equation}\label{equ:relaxed_pi_delta_star}
\begin{aligned}
\pi_{\delta,*} = \argmax_{\pi=(\pi^1, \pi^2)\in\Pi} R(\pi; F_1, F_2)\quad\quad s.t.\quad U(\pi) = 0, \quad S(\pi; F_1, F_2) \leq \delta.\\
\end{aligned}
\end{equation}

Comparing \Cref{equ:pi_star} with \Cref{equ:relaxed_pi_delta_star}, we know that $R(\pi_*)\leq R(\pi_{\delta,*})$. According to \Cref{lemma:small_relaxation_gain}, we further know that $R(\pi_*)\geq R(\pi_{\delta,*})-L\cdot\delta$. Naturally, the substantive unfairness definition is now $\max\{0,S(\pi; F_1, F_2)-\delta\}$. If we still consider this problem under the framework of online learning, then two questions arose naturally: What are the optimal regret rate and (substantive) unfairness rate like? And how can we achieve them simultaneously? From our results in this work, we only know that (1) If $\delta=0$, then both rates are $\Theta(\sqrt{T})$, and (2) if $\delta\geq1$, then the optimal regret is $\Theta(\sqrt{T})$ and the optimal unfairness is $0$ (as it is reduced to the unconstrained pricing problem). In fact, for $\delta= O(\sqrt{1/T})$, we may still achieve $O(\sqrt{T})$ regret and unfairness, but it is not clear if they are always optimal. For $\delta>\sqrt{1/T}$, we conjecture that the optimal regret is still $\Theta(\sqrt{T})$ and the optimal unfairness could be $\Theta(1/(\sqrt{T}\delta))$. 


\paragraph{Optimal Policy on the Continuous Space.} In this work, we restrict our price choices in a fixed price set $\mathbf{V}=\{v_1, v_2, \ldots, v_d\}$ and aims at the optimal distributions on these $v_i$'s. However, if we are allowed to propose any price within $[0,1]$, then the optimal policy could be a tuple of two \emph{continuous} distributions that outperforms any policy restricted on $\mathbf{V}$. Even if we know that customers' valuations are all from $\mathbf{V}$, the optimal policy is not necessarily located inside $\mathbf{V}$ due to the fairness constraints. This optimization problem is even harder than \Cref{equ:pi_star}, and the online-learning scheme further increases its hardness. Existing methods such as continuous distribution discretization \citep{xu2022towards} might work, but would definitely lead to an exponential time complexity.

\section{Conclusion}
\label{sec:conclusion}
In this work, we studied the online pricing problem with fairness constraints. We introduced two fairness notions, a \emph{procedural fairness} and a \emph{substantive fairness} indicating the equality of proposed and accepted prices between two different groups respectively. In order to fulfill these two constraints simultaneously, we adopted \emph{random} pricing policies and established the objective function and rewards in expectation. To solve this problem with unknown demands, we designed a policy-elimination-based algorithm, FPA, that achieves an $\tilde{O}(\sqrt{T})$ regret within an $\tilde{O}(\sqrt{T})$ unfairness. We showed that our algorithm is optimal up to $\log\log T$ factors by proving an $\Omega(\sqrt{T})$ regret lower bound and an $\Omega(\sqrt{T})$ unfairness lower bound for any optimal algorithm with an $O(\sqrt{T})$ regret.



\appendix

\onecolumn
\renewcommand\thesection{\Alph{section}}
\renewcommand\thesubsection{\Alph{section}.\arabic{subsection}}

\begin{appendices}
	\textbf{\huge APPENDIX}

\section{More Related Works}
\label{appendix:more_related_works}
This part serves as a complement to \Cref{sec:related_works}.
\paragraph{Dynamic Pricing}

Single product dynamic pricing problem has been well-studied through \citet{kleinberg2003value, besbes2009dynamic, wang2014close, chen2019nonparametric, wang2021multimodal}. The crux is to learn and approach the optimal of a revenue curve (in continuous price space) from Boolean-censored feedback. \citet{kleinberg2003value} studies this problem under three settings: noise-free, infinitely smooth and stochastic/adversarial valuation assumptions, and proves $\Theta(\log\log{T}), \Theta(\sqrt{T})$ and $\Theta(T^{\frac23})$ minimax regret bounds, sequentially. \citet{wang2021multimodal} further proves a $\Theta(T^{\frac{K+1}{2K+1}})$ minimax regret bound for $K^{\text{th}}$-smooth revenue functions. The key to these problem is to parametrize the revenue curve and estimate these parameters to bound the error. Therefore, a smoothness assumption would reduce the number of local parameters necessary for trading-off the information loss caused by randomness. 

There are recently a variety of works on feature-based dynamic pricing problems \citep{cohen2020feature_journal, leme2018contextual, javanmard2019dynamic, xu2021logarithmic, liu2021optimal, fan2021policy, xu2022towards} where the sellers are asked to sell different products and set prices according to each of their features. All of these literatures listed above adopt a linear feature, i.e., customer's (expected) valuations are linearly dependent on the feature with fixed parameters. In specific, there are a few different problem settings:
\begin{enumerate}
    \item When customer's valuations are deterministic, \citet{cohen2020feature_journal} proposes a binary-search-based algorithm and achieves a $O(\log T)$ regret, which is later improved by \citet{leme2018contextual} with a better $O(d\log\log T)$ regret that matches the information-theoretic lower bound even for the single product pricing problem.
    \item When customers' valuations are linear and noisy. For the setting where the noise distribution is known to the seller, \citet{javanmard2019dynamic} and \citet{xu2021logarithmic} achieves the optimal $O(d\log T)$ regret in stochastic and adversarial settings respectively. \citet{javanmard2019dynamic} further achieves $O(\sqrt{dT})$ regret for unknown but parametric noise distributions, which is also optimal according to \citet{xu2021logarithmic}. For totally unknown noise, \citet{xu2022towards} shows a $O(T^{\frac34})$ regret in the non-continuous case while \citet{fan2021policy} proves a $O(T^{\frac{2K+1}{4K-1}})$ for $K^{\text{th}}$-smooth noises, and both of them are not likely to be optimal (where the current lower bounds are $\Omega(T^{\frac23})$ and $\Omega(T^{\frac{K+1}{2K+1}})$ respectively).
    \end{enumerate}

\paragraph{Fairness in Machine Learning}

Fairness is a long-existing topic that has been extensively studied. In machine learning society, fairness is defined from different perspectives. On the one hand, the concept of \emph{group fairness} requires different groups to receive identical treatment in statistics. In a classification problem, for instance, there are mainly two different types of group fairness: (1) A ``demographic parity'' \citep{dwork2012fairness} that requires the outcome of a classifier to be statistically independent to the group information, and (2) an ``equalized odds'' (including ``equal opportunity'' as a relaxation) \citep{hardt2016equality} that requires the prediction of a classifier to be \emph{conditionally} independent to the group information given the true label. In \citet{agarwal2018reductions}, these probabilistic constraints are further modified as linear constraints, and therefore the fair classification problem is reduced to a cost-sensitive classification problem. It is worth mentioning that \citet{agarwal2018reductions} \emph{allows} an $\epsilon_k$-unfairness due to the learning error and \emph{assumes} $\epsilon_k=O(n^{-\alpha})$ with some $\alpha\leq\frac12$, while we \emph{quantify} the learning-caused unfairness and \emph{upper and lower bound} the cumulative unfairness without pre-assuming its scale. 

On the other hand,  \citep{dwork2012fairness} also proposes the concept of ``individual fairness'' (or ``Lipschitz property'') where the difference of treatments toward two individuals should be upper bounded by a distance metric of their intrinsic features, i.e., $D(\mu_x, \mu_y)\leq d(x, y)$ where $x, y$ are features and $\mu_x, \mu_y$ are the distributions of actions onto $x$ and $y$ respectively. The notion ``time fairness'' is often considered as individual fairness as well. For a more inclusive review on different definitions of fairness in machine learning, please refer to \citet{barocas2017fairness}.

{
\paragraph{Fairness in Online Learning}
Besides existing works on general machine learning fairness, there are some works that study online-learning or bandit problems. This is similar to our setting as we adopt an online pricing process. Among these works, \citet{joseph2016fairness} studies multi-armed and contextual bandits with fairness constraints. Their non-contextual setting is related to our works as our pricing problem can also be treated as a bandit. Their definition of $\delta$-fairness is defined as comparisons among probabilities of taking actions, which is similar to our definition on procedural fairness. However, their fairness definitions are defined from the perspective of arms (i.e. actions): better actions worth larger probability to take. In comparison, our fairness definitions are more on the results: different groups share the same expected prices. \citet{bechavod2020metric} studies an online learning (in specific, an online classification) problem with unknown and non-parametric constraints on individual fairness at each round. They develop an adaptive algorithm that guarantees an $O(\sqrt{T})$ regret as well as an $O(\sqrt{T})$ cumulative fairness loss. However, their problem settings are quite different from ours. Primarily, they assume individual fairness as a constraint, while our fairness definitions are indeed group fairness. Also, their online classification problem is different from our online pricing problem as they have full access to the regret function while we even do not have full-information reward (i.e., which is Boolean-censored). Similarly, their fairness loss is accessible although the unfair pairs of $(\tau_1, \tau_2)$ are not fully accessible, while in our settings we do not know the $S(\pi; F_1, F_2)$ function at all. Besides, we have to satisfy two constraints at one time and one of them (the substantive fairness) is highly non-convex. \citet{gupta2021individual} studies an online learning problem with two different sorts of individual fairness constraints over time: a "fairness-across-time" (FT) and a "fairness-in-hindsight" (FH). They show that it is necessary to have a linear regret under FT constraints, and they also propose a CAFE algorithm that achieves an optimal regret under FH constraints.

Despite the specific properties of fairness constraints, we may also consider the framework of constraint online learning. \citet{yu2017online} studies an online convex optimization (OCO) problem with stochastic constraints, which might be applicable to online fair learning. However, their problem settings and methodologies are largely different from ours: Firstly, their constraints are assumed convex while our substantive fairness constraint (i.e., the $S(\pi; F_1, F_2)$ function) is highly non-convex. Also, they have a direct access to the realized objective function $f^t(x_t)$ at each time while our pricing problem only has a Boolean-censored feedback. More importantly, \citet{yu2017online} assumes the availability of unbiased samples on constraint-related variables. In specific, their constraints are roughly $g_k(x)<0$, and by the end of each period $t$ they receive an unbiased sample of $g_k(x_t)$ for the $x_t$ they have taken. On the contrary, we do not have any unbiased sample of $S(\pi, F_1, F_2)$ at each time, since there is only one customer from one of the two groups. Therefore, we cannot make use of their results in our problem setting.
}
\paragraph{Fairness in Pricing}
There are some works also studying the fairness problem in dynamic pricing besides of the works we discuss in \Cref{sec:related_works}. For example, \citet{richards2016personalized} discusses some fairness issues regarding personalized pricing from the perspective of econometrics. \citet{eyster2021pricing} studies a phenomenon where customers would mistakenly attribute the cost increases to a time unfairness, and they propose methods to release customer's feeling of such unfairness by adjusting prices correspondingly. \citet{chapuis2012price} looked into two fairness concerns called \emph{price fairness} and \emph{pricing fairness}, which indicates the \emph{distributional} and \emph{procedural} fairness of the pricing process respectively, from the \textbf{seller}'s perspective. In fact, their price fairness is more likely to be our \emph{procedural fairness} definition although it is not in their paper. This is because that we are considering the fairness from \textbf{customers}' perspective, where their observations on prices serve as a procedure of their decision process and their decision on whether or not to buy is actually indicating the fairness of results. There are more interesting works as is listed in \Cref{sec:related_works}, and we refer the readers to \citet{chen2021fairness} where there is a more comprehensive review on pricing and fairness.

\section{Proof Details}
\label{appendix:proof_details}
\subsection{Proof of \Cref{theorem:regret_and_unfairness}}
\label{appendix:proof_upper_bound}
\begin{proof}
	First of all, we specify the parameters initialized in \Cref{algorithm}: Let $C_q=3\max\{\frac1q, \frac1{1-q}\}$ and $c_t=\max\{3, \sqrt{\frac3{\hat{F}_{\min}}}\}$.
	For $k=1,2,\ldots$, let $\tau_k = \frac{28C_q}3\cdot d\sqrt{T}\log(\frac{16d\log T}{\epsilon})\cdot2^{k}, \delta_{k,r} = 4c_t\log{\frac{16d\log T}{\epsilon}}d^{\frac32}\sqrt{\frac{C_q}{\tau_k}}, \delta_{k,s}=\frac{32c_t}{(\hat{F}_{\min})^2}\log{\frac{16d\log T}{\epsilon}}d^{\frac32}\sqrt{\frac{C_q}{\tau_k}}$.
	Now we prove that $\hat{F}_{\min}\leq F_{\min}$ with high probability. Recall than $C_q =3\max\{\frac1q, \frac1{1-q}\}$.
	Recall that $\tau_0 = 2\log T\log\frac{16}{\epsilon}$. According to Hoeffding's Inequality, we have:
	\begin{equation}
	\begin{aligned}
	\Pr[|\bar{F}_1(d)-F_1(d)|\geq\frac{F_1(d)}2]\leq&2\exp\{-2(\frac{F_1(d)}2)^2\cdot\frac1{C_q}\tau_0\}\\
	\Leftrightarrow \qquad \Pr[\frac{F_1(d)}2\leq\frac{3F_1(d)}2]\geq&1-2\exp\{-(F_1(d))^2\frac1{C_q}\log T\log{\frac{16}{\epsilon}}\}\\
	\geq&1-\frac\epsilon 8.
	\end{aligned}
	\end{equation}
	Here the last inequality comes from \Cref{assumption:f_min} that $F_1(d)\geq F_{\min}>0$ and therefore we have $(F_{\min})^2\frac1{C_q}\log T\geq1$ with large $T$. Therefore, we have $\frac{F_1(d)}2\bar{F}_1(d)\leq\frac{3F_1(d)}2$ with probability at least $1-\frac\epsilon8$. Similarly, we have $\frac{F_2(d)}2\bar{F}_2(d)\leq\frac{3F_2(d)}2$ with probability at least $1-\frac\epsilon8$. Therefore, with $\Pr\geq1-\frac\epsilon4$, we have $\hat{F}_{\min} = \frac12\min\{\bar{F}_1(d), \bar{F}_2(d)\}\leq\min\{\frac{3F_1(d)}4,\frac{3F_2(d)}4\}=\frac34 F_{\min}< F_{\min}.$

	We define some notations that are helpful to our proof. In epoch $k$, recall that we have:
	\begin{equation}
	\tilde\pi_{k,i,e}=\argmax_{\pi\in\Pi_k} \pi^e(i).
	\end{equation}
	For simplicity, for every $i\in I_k^e$, denote $\rho_{k,e}(i):=\tilde\pi_{k,i,e}^e(i)$ that is the largest probability of choosing price $v_i$ in $G_e$ among all policies in $\Pi_k$. For those $i\notin I_k^e$, we find out the largest $k'$ such that $i\in I_{k'}^e$ and let $\rho_{k,e}(i):=\tilde\pi_{k',i,e}^e(i)$. According to the fact that $\Pi=\Pi_1\supset\Pi_2\supset\ldots\supset\Pi_k\supset\Pi_{k+1}\supset\ldots$, we have
	\begin{equation}
	\label{equ:rho_is_the_max_of_probability}
	\pi^e(i)\leq\rho_{k,e}(i), \forall\pi\in\Pi_k; e=1,2; i\in[d].
	\end{equation}
	
	Next, we prove the following lemmas together by induction over epoch index $k=1,2,\ldots$. We firstly state that 
	\begin{lemma}\label{lemma:pi_star_in_pi_k}
		 Recall the optimal policy $\pi_*$ defined in \Cref{equ:pi_star}. Before Epoch $k$, we have $\pi_*\in\Pi_k$ with high probability (the failure probability will be totally bounded at the end of this proof). 
	\end{lemma}
	which is natural at $k=1$ as $\Pi_1=\Pi$. Now, suppose \Cref{lemma:pi_star_in_pi_k} holds for $\leq k$, then we have:
	\begin{lemma}[Number of Choosing $v_i$ in $G_e$]
		\label{lemma:number_of_choosing_price}
		For $M_{k,e}(i)$ and $N_{k,e}(i)$ defined in \Cref{algorithm}, for any $e=1,2; i\in I_k^e$, with $\Pr\geq 1-\frac\epsilon{2\log T}$ we have:
		\begin{equation}
		\begin{aligned}
		\frac{\rho_{k,e}(i)\cdot\tau_k}{4d\cdot C_q}\leq& M_{k,e}(i),\\
		|N_{k,e}(i)-M_{k,e}(i)\cdot F_e(i)|\leq& c_t\cdot\sqrt{F_e(i)\cdot M_{k,e}(i)}\cdot\log\frac{16d\log T}{\epsilon}.
		\end{aligned}
		\label{equ:number_of_choosing_price}
		\end{equation}
		Here $c_t = \max\{3,\sqrt{\frac3{\hat{F}_{\min}}}\}$.
	\end{lemma}
	\begin{proof}[Proof of \Cref{lemma:number_of_choosing_price}]
		For any $i\in I_k^e$, there exists a policy $\tilde{\pi}_{k,i,e}$ running in Epoch $k$ for at least $\frac{\tau_k}{|A_k|}$ rounds, and. Therefore, we have $\E[M_{k,e}(i)]\geq\rho_{k,e}(i)\cdot\frac{\tau_k}{|A_k|}\cdot\min\{q, 1-q\}\geq\rho_{k,e}(i)\cdot\frac{\tau_k}{|A_k|\cdot C_q}$.
		According to Bernstein's Inequality, for any $e=1,2; i\in I_k^e$ we have:
		
		\begin{equation}
		\begin{aligned}
		&\Pr[|M_{k,e}(i)-\E[M_{k,e}(i)]|\leq\frac{\E[M_{k,e}(i)]}2]\\
		\geq&1-2\exp\{-\frac{\frac12(\frac{\E[M_{k,e}(i)]}{2})^2}{\sum_{t=1}^{\frac{\tau_k}{|A_k|}}\rho_{k,e}(i)\cdot\frac1{C_q}(1-\rho_{k,e}(i)\cdot\frac1{C_q})+\frac13\cdot1\cdot\frac{\E[M_{k,e}(i)]}2}\}\\
		\geq&1-2\exp\{-\frac{\frac18(\E[M_{k,e}(i)])^2}{\rho_{k,e}(i)\cdot\frac{\tau_k}{|A_k|C_q}+\frac16\cdot\E[M_{k,e}(i)]}\}\\
		\geq&1-2\exp\{-\frac{\frac18(\E[M_{k,e}(i)])^2}{\E[M_{k,e}(i)]+\frac16\cdot\E[M_{k,e}(i)]}\}\\
		=&1-2\exp\{-\frac1{8\cdot\frac76}\E[M_{k,e}(i)]\}\\
		\geq&1-2\exp\{-\frac{3}{28}\rho_{k,e}(i)\cdot\frac{\tau_k}{|A_k|\cdot C_q}\}\\
		=&1-2\exp\{-\frac{3}{28}\rho_{k,e}(i)\cdot\frac{ \frac{28}3\cdot d\sqrt{T}\log(\frac{16d\log T}{\epsilon})\cdot2^{k}}{2d\cdot C_q}\}\\
		=&1-2\exp\{-\rho_{k,e}(i)\sqrt{T}\cdot\log(\frac{16d\log T}{\epsilon})\cdot\frac{d\cdot2^k}{2d}\}\\
		\geq&1-2\exp\{-\log(\frac{16d \log T}{\epsilon})\}\\
		=&1-2\cdot\frac{\epsilon}{16d \log T}\\
		=&1-\frac{\epsilon}{8d\log T}.
		\end{aligned}
		\end{equation}
		Here the second line is because that $$\sum_{t=1}^{\tau_k}\E[(\mathbf{1}(\text{choosing }v_i\text{ at time }t))^2]\leq\sum_{t=1}^{\frac{\tau_k}{|A_k|}}\E[(\mathbf{1}(\text{ running }\tilde\pi_{k,i,e}\text{ and } \text{choosing }v_i\text{ at time }t))]$$, the third line is for $1-\rho_{k,e}(i)\cdot\frac1{C_q}\leq 1$, the fourth and sixth line are from $\E[M_{k,e}(i)]\geq\frac{\rho_{k,e}(i)\cdot\tau_k}{|A_k|\cdot C_q}$, the seventh line is by plugging in $\tau_k=\frac{28C_q}3\cdot d\sqrt{T}\log(\frac{16d\log T}{\epsilon})\cdot2^{k}$, the eighth line is equivalent transformation and the ninth line is for $\rho_{k,e}(i)\geq\frac1{\sqrt{T}}$ according to Line 9 of \Cref{algorithm}. As a result, with probability at least $1-\frac{\epsilon}{8d\log T}$, we have 
		\begin{equation}
		\label{equ:m_k_e_lower_bound}
		M_{k,e}(i)\geq\frac{\E[M_{k,e}(i)]}2\geq\frac{\rho_{k,e}(i)\cdot\tau_k}{|A_k|\cdot C_q}\geq\frac{\rho_{k,e}(i)\cdot\tau_k}{4d C_q}.
		\end{equation}
		Now, we analyze $N_{k,e}(i)$ for $i\in I_k^e$. Again, from Line 15 of \Cref{algorithm} we know that $N_{k,e}(i)=\sum_{t=1}^{M_{k,e}(i)}\ind(v_t\text{ is accepted in } G_e)$. Therefore, we apply Bernstein's Inequality and get:
		\begin{equation*}
		\begin{aligned}
		&\Pr[|N_{k,e}(i) - M_{k,e}(i)\cdot F_e(i)|\geq c_t\cdot\sqrt{M_{k,e}(i)\cdot F_e(i)}\log{\frac{16d\log T}\epsilon}]\\
		\leq&2\exp\{-\frac{\frac12 c_t^2\cdot M_{k,e}(i)F_e(i)(\log{\frac{16d\log T}\epsilon})^2}{M_{k,e}(i)F_e(i)\log{\frac{16d\log T}\epsilon}(1-F_e(i))+\frac13\cdot(c_t\cdot\sqrt{M_{k,e}(i)\cdot F_e(i)}\log{\frac{16d\log T}\epsilon})} \}\\
		\leq&2\exp\{-\frac{\frac12 c_t^2\log\frac{16d\log T}\epsilon}{1+\frac{c_t}3}\}\\
		\leq&\frac\epsilon {8\log T}.
		\end{aligned}
		\end{equation*}
		Here the last line is by $c_t=\max\{3, \sqrt{\frac3{\hat{F}_{\min}}}\}\geq3$ and therefore $\frac{\frac12c_t^2}{1+\frac{c_t}3}\geq 1$. As a result, for $e=1,2; i\in I_k^e$, with $\Pr\geq 1-\frac\epsilon {8\log T}$ we have
		\begin{equation*}
		|N_{k,e}(i)-M_{k,e}(i)\cdot F_e(i)|\leq c_t\cdot\sqrt{M_{k,e}(i)F_e(i)}\log\frac{16d\log T}\epsilon.
		\end{equation*}
		That is to say,
		\begin{equation*}
		\begin{aligned}
		|\bar{F}_{k,e}(i)-F_e(i)|&=|\max\{\frac{N_{k,e}(i)}{M_{k,e}(i)}, \hat{F}_{\min}\}-F_e(i)|\\
		&\leq|\frac{N_{k,e}(i)}{M_{k,e}(i)}|\\
		&\leq c_t\cdot\sqrt{\frac{F_e(i)}{M_{k,e}(i)}}\log\frac{16d\log T}\epsilon\\
		&\leq c_t\log\frac{16d\log T}\epsilon\sqrt{F_e(i)}\cdot\sqrt{\frac{4d C_q}{\rho_{k,e}(i)\tau_k}}.
		\end{aligned}
		\end{equation*}
		Here the first line is by definition of $\bar{F}_{k,e}$, the second line is because $\hat{F}_{\min}\leq F_{\min}\leq F_e(i)$, the third line is by the inequality above and the last line is by \Cref{equ:m_k_e_lower_bound}.
		Therefore, with $\Pr\geq 1-\frac\epsilon{2\log T}$, \Cref{equ:number_of_choosing_price} holds for $e=1,2$ and for $\forall i\in I_k^e$.
	\end{proof}
	Given \Cref{lemma:number_of_choosing_price}, we have the following corollary directly:
	\begin{corollary}[Estimation Error of $\bar{F}_{k,e}(i)$]
		\label{corollary:estimation_error_f}
		Assume that \Cref{lemma:number_of_choosing_price} holds. For $\bar{F}_{k,e}(i) = \max\{\frac{N_{k,e}(i)}{M_{k,e}(i)}, \hat{F}_{\min}\}$ defined in \Cref{algorithm}, for any $e=1,2; i\in I_k^e$, we have:
		\begin{equation}
		|\bar{F}_{k,e}(i)-F_e(i)|\leq c_t\cdot\log\frac{16d\log T}{\epsilon}\sqrt{\frac{4dF_e(i)C_q}{\rho_{k,e}(i)\tau_k}}.
		\label{estimation_error_f}
		\end{equation}
	\end{corollary}
	For simplicity, denote$R(\pi):=R(\pi, F_1, F_2)$, $S(\pi):=S(\pi, F_1, F_2)$, $\hat{R}_k(\pi):=R(\pi, \bar{F}_{k,1}, \bar{F}_{k,2})$ and $\hat{S}_k(\pi):=S(\pi, \bar{F}_{k,1}, \bar{F}_{k,2})$.
	Based on \Cref{corollary:estimation_error_f}, we can bound the estimation error of $\hat{R}_k(\pi)$ and $\hat{S}_k(\pi)$ by the following lemma:
	\begin{lemma}[Estimation Error of $R$ and $S$ Functions]
		\label{lemma:estimation_function_error}
		Given \Cref{lemma:number_of_choosing_price}, we have:
		\begin{equation}
		\label{equ:estimation_function_error}
		\begin{aligned}
		|R(\pi)-\hat{R}_k(\pi)|\leq& \frac{\delta_{k,r}}2,\\
		|S(\pi)-\hat{S}_k(\pi)|\leq& \frac{\delta_{k,s}}2.
		\end{aligned}
		\end{equation}
	\end{lemma}
	Here $\delta_{k,r} = 4c_t\log{\frac{16d\log T}{\epsilon}}d^{\frac32}\sqrt{\frac{C_q}{\tau_k}}$ and $\delta_{k,s}=\frac{32c_t}{\hat{F}_{\min}^2}\log{\frac{16d\log T}{\epsilon}}d^{\frac32}\sqrt{\frac1{\tau_k}}$ as is defined in \Cref{theorem:regret_and_unfairness}. 
	\begin{proof}[Proof of \Cref{lemma:estimation_function_error}]
		First of all, we show that for any $e=1,2;i=1,2,\ldots, d$ and for any $\pi\in\Pi_k$,
		\begin{equation}
		\label{equ:error_f_times_pi_upper_bounded}
		|\bar{F}_{k,e}(i)-F_e(i)|\cdot\pi^e(i) \leq c_t\cdot\log\frac{16d\log T}\epsilon\sqrt{\frac{4d C_q}{\tau_k}}.
		\end{equation}
		In fact, when $i\in I_k^e$, according to \Cref{lemma:number_of_choosing_price} we have
		\begin{equation}\label{equ:f_estimation_error_for_large_probability}
		\begin{aligned}
		|\bar{F}_{k,e}(i)-F_e(i)|\cdot\pi^e(i)\leq&|\bar{F}_{k,e}(i)-F_e(i)|\cdot\rho_{k,e}(i)\\
		\leq& c_t\cdot\log\frac{16d\log T}{\epsilon}\sqrt{\frac{4dF_e(i)C_q}{\rho_{k,e}(i)\tau_k}}\cdot\rho_{k,e}(i)\\
		\leq& c_t\cdot\log\frac{16d\log T}\epsilon\sqrt{\frac{4dC_q\cdot(\rho_{k,e}(i))}{\tau_k}}\\
		\leq& c_t\cdot\log\frac{16d\log T}\epsilon\sqrt{\frac{4dC_q}{\tau_k}}.
		\end{aligned}
		\end{equation}
		When $i\notin I_k^e$, we know that $\rho_{k,e}(i)\leq\frac1{\sqrt{T}}$ and thus $\pi^e(i)\leq\frac1{\sqrt{T}}, \forall\pi\in\Pi_k$ according to \Cref{equ:rho_is_the_max_of_probability}. Also, since $\pi_*\in\Pi_k$ by induction, we know that $\pi_*^e\leq\rho_{k,e}(i)\leq\frac1{\sqrt{T}}$. Therefore, we have
		\begin{equation}
		\label{equ:f_estimation_error_for_small_probability}
		\begin{aligned}
		|\bar{F}_{k,e}(i)-F_e(i)|\cdot\pi^e(i)\leq&|\bar{F}_{k,e}(i)-F_e(i)|\cdot\rho_{k,e}(i)\\
		\leq&|\bar{F}_{k,e}(i)\pi^e(i)-F_e(i)|\cdot\frac1{\sqrt{T}}\\
		\leq&1\cdot\frac1{\sqrt{T}}\\
		\leq&c_t\cdot\log\frac{16d\log T}\epsilon\sqrt{\frac{4dC_q}{\tau_k}}.
		\end{aligned}
		\end{equation}
		Here we assume that $\log\frac{16d\log T}\epsilon>1$ without losing of generality (i.e., $T$ is sufficiently large and $\epsilon$ can be arbitrarily close to zero), and the last inequality comes from $c_t\geq3>1$ and $4d\geq 1$ and $\tau_k\leq T$. Combining \Cref{equ:f_estimation_error_for_large_probability} and \Cref{equ:f_estimation_error_for_small_probability}, we know that \Cref{equ:error_f_times_pi_upper_bounded} holds for all $e=1,2; i\in[d]$. 
		Remember that $R(\pi)=q\cdot\sum_{i=1}^d F_1(i)\pi^1(i) + (1-q)\cdot\sum_{j=1}^dF_2(j)\pi^2(j)$ and that $\hat{R}_k(\pi)=q\cdot\sum_{i=1}^d \bar{F}_{k,1}(i)\pi^1(i) + (1-q)\cdot\sum_{j=1}^d\hat{F}_{k,2}(j)\pi^2(j)$. Therefore, we may bound the error between $\hat{R}_k(\pi)$ and $R(\pi)$. For $\forall\pi\in\Pi_k$, we have
		\begin{equation}
		\label{equ:error_bound_of_r_hat}
		\begin{aligned}
		|\hat{R}_k(\pi)-R(\pi)|&\leq q\cdot\sum_{i=1}^d|\bar{F}_{k,1}(i)-F_1(i)|\cdot\pi^1(i) + (1-q)\cdot\sum_{j=1}^d|\bar{F}_{k,2}(j)-F_2(j)|\cdot\pi^2(i)\\
		&\leq q\cdot\sum_{i=1}^d c_t\cdot\log\frac{16d\log T}\epsilon\sqrt{\frac{4d C_q}{\tau_k}} + (1-q)\cdot\sum_{j=1}^dc_t\cdot\log\frac{16d\log T}\epsilon\sqrt{\frac{4d C_q}{\tau_k}}\\
		&= c_t\cdot\log\frac{16d\log T}\epsilon\sqrt{\frac{4d C_q}{\tau_k}}\cdot d\\
		&=\frac{\delta_{k,r}}2.
		\end{aligned}
		\end{equation}
		Similarly, for the error between $\hat{S}_k(\pi)$ and $S(\pi)$ as $\pi\in\Pi_k$, we have:

		\begin{equation}
		\label{equ:error_s}
		\begin{aligned}
		|\hat{S}_k(\pi)-S(\pi)|=&||\frac{\vv^{\top}\hat{F}_{k,1}\pi^1}{\ind^{\top}\hat{F}_{k,1}\pi^1}-\frac{\vv^{\top}\hat{F}_{k,2}\pi^2}{\ind^{\top}\hat{F}_{k,2}\pi^2}|-|\frac{\vv^{\top}F_1\pi^1}{\ind^\top F_1\pi^1}-\frac{\vv^\top F_2\pi^2}{\ind^\top F_2\pi^2}||\\
		\leq&|\frac{\vv^{\top}\hat{F}_{k,1}\pi^1}{\ind^{\top}\hat{F}_{k,1}\pi^1}-\frac{\vv^{\top}\hat{F}_{k,2}\pi^2}{\ind^{\top}\hat{F}_{k,2}\pi^2}-(\frac{\vv^{\top}F_1\pi^1}{\ind^\top F_1\pi^1}-\frac{\vv^\top F_2\pi^2}{\ind^\top F_2\pi^2})|\\
		\leq&|\frac{\vv^{\top}\hat{F}_{k,1}\pi^1}{\ind^{\top}\hat{F}_{k,1}\pi^1}-\frac{\vv^{\top}F_1\pi^1}{\ind^\top F_1\pi^1}|+|\frac{\vv^{\top}\hat{F}_{k,2}\pi^2}{\ind^{\top}\hat{F}_{k,2}\pi^2}-\frac{\vv^\top F_2\pi^2}{\ind^\top F_2\pi^2}|\\
		=&\sum_{e=1}^2|\frac{\vv^{\top}\hat{F}_{k,e}\pi^e}{\ind^{\top}\hat{F}_{k,e}\pi^e}-\frac{\vv^{\top}F_e\pi^e}{\ind^\top F_e\pi^e}|\\
		=&\sum_{e=1}^2|\frac{(\vv^{\top}\hat{F}_{k,e}\pi^e)(\ind^\top F_e\pi^e)-(\vv^{\top}F_e\pi^e)(\ind^{\top}\hat{F}_{k,e}\pi^e)}{(\ind^{\top}\hat{F}_{k,e}\pi^e)(\ind^\top F_e\pi^e)}|\\
		=&\sum_{e=1}^2\frac{|(\pi^e)^{\top}(\hat{F}_{k,e}-F_e)\vv\cdot(\ind^\top F_e\pi^e)+(\vv^{\top}F_e\pi^e)\ind^\top (F_e-\hat{F}_{k,e})\pi^e|}{|(\ind^{\top}\hat{F}_{k,e}\pi^e)|\cdot|(\ind^\top F_e\pi^e)|}\\
		\leq&\sum_{e=1}^2\frac{(\ind^\top F_e\pi^e)\cdot\sum_{i=1}^d\pi^e(i)(\bar{F}_{k,e}(i)-F_e(i))v_i + (\vv^{\top}F_e\pi^e)\cdot\sum_{j=1}^d1\cdot(F_e(j)-\hat{F}_{k,e}(j))\cdot\pi^e(j)}{|\hat{F}_{\min}|\cdot|\hat{F}_{\min}|}\\
		\leq&\sum_{e=1}^2\frac{1\cdot\sum_{i=1}^d\pi^e(i)|\bar{F}_{k,e}(i)-F_e(i)|\cdot1 + 1 \cdot \sum_{j=1}^d1\cdot|F_e(j)-\hat{F}_{k,e}(j)|\pi^e(j)}{(\hat{F}_{\min})^2}\\
		\leq&\frac1{(\hat{F}_{\min})^2}\sum_{e=1}^2(\sum_{i=1}^dc_t\cdot\log\frac{16d\log T}\epsilon\sqrt{\frac{4d C_q}{\tau_k}} + \sum_{j=1}^dc_t\cdot\log\frac{16d\log T}\epsilon\sqrt{\frac{4d C_q}{\tau_k}})\\
		=&\frac1{(\hat{F}_{\min})^2}\cdot2d\cdot{}c_t\cdot\log\frac{16d\log T}\epsilon\sqrt{\frac{4d C_q}{\tau_k}}\\
		\leq&\frac{\delta_{k,s}}2.
		\end{aligned}
		\end{equation}
	\end{proof}
	Since we have $\hat{S}_k(\pi)\leq\delta_{k,s}, \forall\pi\in\Pi_{k+1}$ by definition in \Cref{algorithm}, we know that for any policy $\pi\in\Pi_{k+1}$,
	\begin{equation}
	\label{equ:bound_unfairness}
	\begin{aligned}
	S(\pi) = & \hat{S}_k(\pi) + (S(\pi)-\hat{S}_k(\pi))\\
	\leq & \hat{S}_k(\pi) + |S(\pi) - \hat{S}_k(\pi)|\\
	\leq & \delta_{k,s} + \frac{\delta_{k,s}}2\\
	\leq & 2\delta_{k,s}.
	\end{aligned}
	\end{equation}
	Therefore, any policy remaining in $\Pi_{k+1}$ suffers at most $2\delta_{k,s}$ unfairness. Now let us bound the regret of any policy in $\Pi_{k+1}$. Here we firstly propose a lemma.
	\begin{lemma}[Small Relaxation Gain]\label{lemma:small_relaxation_gain}
		Recall that $\pi_*$ is the solution to \Cref{equ:pi_star}. Define a $\pi_{\delta,*}$ as follows.
		\begin{equation}
		\label{equ:def_delta_relaxation}
		\begin{aligned}
		\pi_{\delta, *} =& \argmax_{\pi\in\Pi}R(\pi)\\
		s.t.& \qquad S(\pi)\leq\delta.
		\end{aligned} 
		\end{equation}
		Then there exists a constant $L\in\R_+$ such that $R(\pi_{\delta, *})-R(\pi_*)\leq \frac{L}2\cdot\delta$.
	\end{lemma}
	We leave the proof of Lemma \ref{lemma:small_relaxation_gain} to the end of this section.
	Given \Cref{lemma:small_relaxation_gain} and the previous \Cref{lemma:estimation_function_error}, we have:
	\begin{equation}\label{equ:pistar_in_pi+1}
	\begin{aligned}
	\hat{R}(\hat{\pi}_{k,*})-\hat{R}(\pi_*)=&\hat{R}(\hat{\pi}_{k,*})-R(\hat{\pi}_{k,*})+R(\hat{\pi}_{k,*})-R(\pi_{2\delta_{k,s}, *}) + R(\pi_{2\delta_{k,s}}, *)-R(\pi_*) + R(\pi_*)-\hat{R}(\pi_*)\\
	\leq&|\hat{R}(\hat{\pi}_{k,*})-R(\hat{\pi}_{k,*})|+(R(\hat{\pi}_{k,*})-R(\pi_{2\delta_{k,s}, *}))+(R(\pi_{2\delta_{k,s}}, *)-R(\pi_*))+|R(\pi_*)-\hat{R}(\pi_*)|\\
	\leq&\frac{\delta_{k,r}}2+0+\frac{L}2\cdot 2\delta_{k,s} + \frac{\delta_{k,r}}2\\
	=&\delta_{k,r} + L\cdot\delta_{k,s}
	\end{aligned}
	\end{equation}
	
	By definition of $\Pi_{k+1}$ at \Cref{equ:pi_k+1}, we know that $\pi_*\in\Pi_{k+1}$, which holds \Cref{lemma:pi_star_in_pi_k} at $k+1$ and therefore completes the induction. As a result, all \Cref{lemma:pi_star_in_pi_k}, \Cref{lemma:number_of_choosing_price}, \Cref{lemma:estimation_function_error} and \Cref{lemma:small_relaxation_gain} holds for all $k=1,2,\ldots$. As a result, we may calculate the total regret and substantive unfairness as follows.
	
	For the regret, we may divide the whole time horizon $T$ into three stages:
	\begin{enumerate}
		\item Stage 0: Before epochs where we propose $v_d$ for $\tau_0 = 2\log{T}\log{\frac{16}{\epsilon}}$ rounds in either $G_1$ or $G_2$. The regret for this stage is $O(\log{T}\log{\frac1\epsilon})$.
		\item Stage 1: Epoch 1 where we try every price for $2\cdot\frac{\tau_1}{2d}$ rounds in either $G_1$ or $G_2$. The regret for this stage is $O(\tau_1)=O(d\sqrt{T}\log\frac{\log T}{\epsilon})$.
		\item Stage 2: Epoch $k=2,3,\ldots$. In each epoch $k$, every policy $\pi$ we run satisfies $\pi\in\Pi_k$. Therefore, for any $\pi$ running in Epoch $k=2,3,\ldots$, we have
		\begin{equation}
		\label{equ:regret_perround}
		\begin{aligned}
		R(\pi_*)-R(\pi)=& (R(\pi_*)-\hat{R}_{k-1}(\pi_*))+(\hat{R}_{k-1}(\pi_*)-\hat{R}_{k-1}(\pi))+(\hat{R}_{k-1}(\pi)-R(\pi))\\
		\leq & \frac{\delta_{k-1,r}}2 + (\delta_{k-1,r} + L\cdot\delta_{k-1,s}) + \frac{\delta_{k-1,r}}2\\
		=& 2\delta_{k-1, r} + L\cdot\delta_{k-1,s}.
		\end{aligned}
		\end{equation}
		The second line is by definition of $\Pi_k$ for $k\geq2$ and by \Cref{lemma:estimation_function_error}. Suppose there are $K$ epochs in total, and then we know that:
		\begin{equation*}
		T \geq \sum_{k=1}^{K}\tau_k = \frac{28C_q}3\cdot d\sqrt{T}\cdot\log\frac{16d\log T}\epsilon\cdot\sum_{k=1}^{K}2^k.
		\end{equation*}
		Solve the equaltion above and we get $K=O(\log\frac{\sqrt T}{d\log{\frac{d\log T}\epsilon}})$ and $K\leq\frac12\log T$. Therefore, the total regret of Stage 2 is
		\begin{equation}
		\label{equ:total_regret}
		Reg=O(\sum_{k=2}^K\tau_k\cdot (2\delta_{k-1, r} + L\cdot\delta_{k-1,s})) = O(\sqrt{T}\cdot d^{\frac32}\log\frac {d\log T}\epsilon)
		\end{equation}
	\end{enumerate}
	Add the regret of all three stages above, we get that the total regret is $O(\sqrt{T}\cdot d^{\frac32}\log\frac {d\log T}\epsilon)$.
	
	For the unfairness, we derive it similarly in three stages:
	\begin{enumerate}
		\item Stage 0: Before epochs where we propose $v_d$ for $\tau_0 = 2\log{T}\log{\frac{16}{\epsilon}}$ rounds in either $G_1$ or $G_2$. The unfairness for this stage is $0$ as we always propose the same price to both groups.
		\item Stage 1: Epoch 1 where we try every price for $2\cdot\frac{\tau_1}{2d}$ rounds in either $G_1$ or $G_2$. The regret for this stage is $0$ as well.
		\item Stage 2: Epoch $k=2,3,\ldots$. In each epoch $k$, every policy $\pi$ we run satisfies $\pi\in\Pi_k$. Therefore, for any $\pi$ running in Epoch $k=2,3,\ldots$, we have
		\begin{equation}
		\label{equ:unfairness_perround}
		\begin{aligned}
		S(\pi)=& \hat{S}_{k-1}(\pi) +(S(\pi)-\hat{S}_{k-1}(\pi))\\
		\leq& \delta_{k-1,s} + \frac{\delta_{k-1,s}}2\\
		\leq& \frac{3\delta_{k-1,s}}2.
		\end{aligned}
		\end{equation}
		Here the last line is by definition of $\Pi_k$ for $k\geq2$ and by \Cref{lemma:estimation_function_error}. Therefore, the total unfairness of Stage 2 is
		\begin{equation}
		\label{equ:total_unfairness}
		Unf \leq \sum_{k=2}^K \tau_k\cdot \frac{3\delta_{k-1,s}}2=O(\sqrt{T}\cdot d^{\frac32}\log\frac {d\log T}\epsilon).
		\end{equation}
	\end{enumerate} 
	Therefore, the total substantive unfairness of all three stages is $O(\sqrt{T}\cdot d^{\frac32}\log\frac {d\log T} \epsilon)$ as well. 
	
	Finally, we count the probability of failure of all stages. For Stage 0, the failure probability is $\Pr_0\leq \frac\epsilon4$. For each epoch, the failure probability is $\Pr_k \leq \frac\epsilon{2\log T}$. Since there are $K\leq\frac{\log T}2$ epochs, the total failure probability is $\Pr_{failure}\leq\Pr_0 + K\cdot\Pr_k\leq\frac\epsilon4+K\cdot\frac\epsilon{2\log T}\leq \frac\epsilon2<\epsilon$. That is to say, \Cref{theorem:regret_and_unfairness} holds with probability at least $\Pr\geq1-\epsilon$.
\end{proof}
At the end of this subsection, we prove \Cref{lemma:small_relaxation_gain} as we promised above.
\begin{proof}[Proof of \Cref{lemma:small_relaxation_gain}]
	Denote any policy $\pi\in\Pi$ as $\pi=(\pi^1, \pi^2)$. For the simplicity of notation, we denote the following functions:
	\begin{enumerate}[label=(\alph*)]
		\item Define $R_1(\pi^1) = \vv^{\top}F_1\pi^1$;
		\item Define $R_2(\pi^2) = \vv^{\top}F_2\pi^2$;
		\item Define $S_1(\pi^1) = \frac{\vv^{\top}F_1\pi^1}{\ind^{\top}F_1\pi^1}$;
		\item Define $S_2(\pi^2) - \frac{\vv^{\top}F_2\pi^2}{\ind^{\top}F_2\pi^2}$.
	\end{enumerate}
	For $\pi_{\delta, *}$ defined in \Cref{equ:relaxed_pi_delta_star}, denote $V_s:=S_1(\pi^1_{\delta, *})$ and $z=S_2(\pi^2_{\delta, *})-V_s$. Therefore, we know that $V_s\in[v_1, 1]$ (recalling that $v_1>0$) and $z\in[-\delta, \delta]$. According to the optimality of $\pi_{\delta, *}$, we have:
	\begin{equation}
	\label{equ:optimality_delta}
	\begin{aligned}
	\pi_{\delta, *}&=\argmax_{\pi\in\Pi, V_s\in[v_1, 1], z\in[-\delta, \delta]} qR_1(\pi^1) +(1-q)R_2(\pi^2)\\
	s.t.&\quad S_1(\pi^1) = V_s\\
	&\quad S_2(\pi^2) = V_s + z
	\end{aligned}
	\end{equation}
	. Consider the constraint $S_2(\pi^2)- V_s \in [-\delta, \delta]$, we can derive the following relaxation:
	\begin{equation}
	\label{equ:delta_relaxation_fraction}
	\begin{aligned}
	&S_2(\pi^2)- V_s \in [-\delta, \delta]\\
	\Leftrightarrow\qquad-\delta&\leq\frac{\vv^{\top}F_2\pi^2}{\ind^{\top}F_2\pi^2}-V_s\leq\delta\\
	\Rightarrow\qquad-\delta(\ind^{\top}F_2\pi^2)&\leq\vv^{\top}F_2\pi^2-V_s\cdot\ind^{\top}F_2\pi^2\leq\delta(\ind^{\top}F_2\pi^2)\\
	\Rightarrow\qquad-\delta&\leq\vv^{\top}F_2\pi^2-V_s\cdot\ind^{\top}F_2\pi^2\leq\delta.\\
	\end{aligned}
	\end{equation}
	This is because $\ind^{\top}F_2\pi^2\in[F_{\min}, 1]\subset(0,1]$. Therefore, we may define $\theta_{\delta}=(\theta_{\delta}^1, \theta_{\delta}^2)\in\Pi$ such that
	\begin{equation}
	\label{equ:relaxation_step_1}
	\begin{aligned}
	\theta_{\delta}:=\argmax_{\theta\in\Pi, r, w\in[v_1\cdot F_{\min}, 1]}& q R_1(\theta^1)+(1-q) R_2(\theta^2)\\
	s.t.\qquad \vv^{\top}F_1\theta^1 =& w\\
	\qquad \cdot\ind^{\top}F_1\theta^1 = &\frac w {V_s}\\
	\qquad \vv^{\top}F_2\theta^2 =& r\\
	-\delta\leq v_1 \cdot\ind^{\top}F_2\theta^2-\frac{r\cdot v_1}{V_s}\leq&\delta,
	\end{aligned}
	\end{equation}
	for any $\theta\geq0$. Here we make use of the fact that $V_s\in[v_1,1]$. Notice that $[v_1\cdot F_{\min}, 1]$ contains all possible $r$'s due to the fact that $F_e(i)>F_{\min}$ and $v_i\geq v_1$ for any $i\in[d]$ and $e\in\{1,2\}$, then we have $R_2(\theta_{\delta}^2)\geq R_2(\pi_{\delta, *}^2)$ as a relaxation of conditions, which means that $R(\theta_{\delta})\geq R(\pi_{\delta, *})$. Consider another policy $\pi_{start}$:
	\begin{equation}
	\label{equ:pi_start}
	\begin{aligned}
	\pi_{start}:=\argmax_{\pi\in\Pi} qR_1(\pi^1) &+(1-q)R_2(\pi^2)\\
	s.t.\qquad S_1(\pi^1)  =& V_s\\
	\qquad S_2(\pi^2) = & V_s.
	\end{aligned}
	\end{equation}
	Therefore, we know that when $\delta=0$, we have $\theta_0=\pi_{start}$ exactly. Also, since $\pi_*$ can also be defined as follows:
	\begin{equation}
	\label{equ:pi_star_another_definition}
	\begin{aligned}
	\pi_* =\argmax_{\pi\in\Pi, v_s\in[v_1, 1]} qR_1(\pi^1) +& (1-q) R_2(\pi^2)\\
	s.t.\qquad S_1(\pi^1) = & v_s\\
	\qquad S_2(\pi^2) = & v_s.
	\end{aligned}
	\end{equation}
	According to the optimality of $\pi_*$ over all $v_s\in[v_1, 1]$ while $\pi_{start}$ is restricted on a specific $V_s$, we have $R(\pi_*)\geq R(\pi_{start})=R(\theta_0)$. Recall that we also have $R(\theta_{\delta})\geq R(\pi_{\delta, *})$. Therefore, as long as we show that there exists a constant $L$ such that $R(\theta_{\delta})-R(\theta_0)\leq\frac L 2\cdot\delta$, then it is sufficient to show that $R(\pi_{\delta, *})-R(\pi_*)\leq\frac L 2\cdot\delta$.
	
	Denote
	\begin{equation}
	\label{equ:def_tilde_theta_delta_equation}
	\tilde{\theta}_{\delta}=[(\theta_{\delta}^1)^{\top}, w, \frac{w\cdot V_1}{V_s}, (\theta_{\delta}^2)^{\top}, r, \frac{r\cdot v_1}{V_s}]^{\top}\in\R^{2d+4}.
	\end{equation}
	Of course $\|\tilde{\theta}_{\delta}\|_1\leq 1+w+\frac{w}{V_s}+1+r+\frac{r}{V_s}\leq 4+2\frac2{v_1}$. Denote the domain of $\tilde{\theta}{\delta}$ as $\cD(\delta)$. Therefore, we know that for any $\theta\in\cD(\delta)$, we have
	\begin{equation}
	\begin{aligned}
	\theta &\succeq 0\\
	[\ind_d^{\top}, 0, 0,\ldots, 0]\theta &= 1\\
	[0,\ldots, 0,0,0,\ind_d^{\top}, 0,0]\theta &=1\\
	[0,\ldots, 0, 1, 0, \ldots, 0]\theta &\leq 1 \text{ (for } 1 \text{ in the }(d+1)^{th}\text{ place)}\\
	[0,\ldots, 0, 1, 0, \ldots, 0]\theta &\leq  1 \text{ (for } 1 \text{ in the }(d+2)^{th}\text{ place)}\\
	[0,\ldots, 0, 1, 0]\theta &\leq 1\\
	[0,\ldots, 0, 0, 1]\theta &\leq  1\\
	\end{aligned}
	\label{equ:tilde_theta_range}
	\end{equation}
	Denote $\tilde{\cD}(\delta)$ as the space of all $\theta$ satisfying \Cref{equ:tilde_theta_range}, and we know that $\tilde{\cD}(\delta)\supseteq\cD(\delta)$ and $\tilde{\cD}(\delta)$ is a bounded, close and convex set with only linear boundaries.  
	Also, denote the following fixed parameters:
	\begin{equation}
	\label{equ:def_general_parameter}
	\begin{aligned}
	\mathbf{a}&:=[q\cdot(\vv^{\top}F_1), 0, 0, (1-q)\cdot(\vv^{\top}F_2), 0, 0]\in\R^{2d+4}\\
	\mathbf{b}_1&:=[\vv^{\top}F_1, -1, 0, 0, \ldots, 0]\in\R^{2d+4}\\
	\mathbf{b}_2&:=[v_1\ind^{\top}F_1, 0, -1, 0,0,\ldots, 0]\in\R^{2d+4}\\
	\mathbf{g}&:=[0,\ldots,0,0,0,\vv^{\top}F_2, -1, 0]\in\R^{2d+4}\\
	\mathbf{d}&:=[0,\ldots,0,0,0,v_1\cdot\ind^{\top}F_2, 0, -1]\in\R^{2d+4}.\\
	\end{aligned}
	\end{equation}
	Again, these parameters are all constants under the same problem setting. Given these parameters, for the definition of $\theta_{\delta}$ in \Cref{equ:relaxation_step_1}, we may transform that definition into the following one equivalently:
	\begin{equation}
	\label{equ:def_theta_delta_second_definition}
	\begin{aligned}
	\tilde{\theta}_{\delta}:=\argmax_{\theta\in\tilde{\cD}(\delta)}\ & \mathbf{a}^{\top}\theta\\
	s.t.\quad \mathbf{b}_1^{\top}\theta =& 0\\
	\mathbf{b}_2^{\top}\theta =& 0\\
	\mathbf{g}^{\top}\theta = & 0\\
	\mathbf{d}^{\top}\theta \in& [-\delta, \delta].
	\end{aligned}
	\end{equation}
	Since $\tilde{\cD}(\delta)\supseteq\cD(\delta)$, we know that $\mathbf{a}^{\top}\tilde{\theta}_{\delta}\geq R(\theta_{\delta})$. Denote
	\begin{equation*}
	\tilde{\cD}_{abg}(\delta):=\{\theta|\theta\in\tilde{\cD}(\delta), \mathbf{b}_1^{\top}\theta=0, \mathbf{b}_2^{\top}\theta=0,\mathbf{g}^{\top}\theta=0\},
	\end{equation*}
	and we know that $\tilde{\cD}_{abg}(\delta)$ is also a bounded close and convex set with only linear boundaries. Therefore, \Cref{equ:def_theta_delta_second_definition} is equivalent to the following definition:
	\begin{equation}
	\label{equ:def_theta_delta_new_definition}
	\begin{aligned}
	\tilde{\theta}_{\delta}:=\argmax_{\theta\in\tilde{\cD}_{abg}(\delta)}\ & \mathbf{a}^{\top}\theta\\
	s.t.\qquad\mathbf{d}^{\top}\theta\in[-\delta, \delta].
	\end{aligned}
	\end{equation}
	Now we present the following lemma.
	\begin{lemma}[Bounded Shifting]
		Given any  space $Q\subset\R^{n}$ that is bounded, close and convex with only linear boundaries, consider the following subset $Q_0:=\{x\in Q, d^{\top}x=0\}\neq\emptyset$. Then there exists a constant $C_L$ such that for any $z\in\R$, $Q_{z}:=\{x\in Q, d^{\top}x=z\}$ and any $\theta_z\in Q_{z}$, there always exists a $\theta_0\in Q_0$ such that $\|\theta_z-\theta_0\|_2\leq C_L\cdot|z|$.
		\label{lemma:bounded_shifting}
	\end{lemma}

	\begin{proof}[Proof of \Cref{lemma:bounded_shifting}]
		Without loss of generality, we assume that $z>0$. Denote $Q^+=Q\cap\{x:d^\top x\geq 0\}$. Because $Q$ is bounded, close and convex with only linear boundaries, the number of vertex of $Q^+$ must be finite. The vertex set of $Q^+$ can be decomposed as $V=V_0+V_1$, where $V_0$ denotes the vertex such that $d^\top x=0$ while $V_1$ denotes the vertex such that $d^\top x>0$. In addition, $Q_0=Q^+\cap \{x:d^\top x=0\}$ is the cross section while we define $B=\{x:d^\top x=0\}$.
		
		For each point $x\in Q_0$, we define $\beta_x$ to be $\min\{\text{The intersection angle between $B$ and}\,\,\overrightarrow{xv},\,\,v\in V_1\}$. Due to the fact that $\beta_x$ is continuous upon $x$, $\beta_x>0$ and the domain $Q_0$ is bounded and close, there exists a $\beta_{\min}>0$ such that $\beta_x\geq \beta_{\min}$, $\forall x\in Q_0$. Then we construct a corresponding cone $Cone_x$ for each $x\in Q_0$ such that $Cone_x=\{v:d^\top v\geq0,\text{and the intersection angle between $B$ and}\,\,\overrightarrow{xv}\geq \beta_{\min}\}$.
		
		Since $Q^+$ is bounded, close and convex with only linear boundaries, for any point $\theta_z\in Q_z$, there exits $v_1,v_2,\cdots,v_k$ and $a_1,a_2,\cdots,a_k$ such that $v_i\in V, a_i\geq 0,\forall i\in[k]$ and $\sum_{i=1}^{k} a_i=1$ and it holds that $\theta_z=\sum_{i=1}^k a_i v_i$. Then according to our construction of the cones, for each selected vertex $v_i$, there exists a cone $Cone_{t_i}$ such that $t_i\in Q_0$ and $v_i\in Cone_{t_i}$. We claim that $Cone_{\sum_{i=1}^{k} a_i t_i}=\{\sum_{i=1}^{k}a_i f_i:f_i\in Cone_{t_i}\}$. Therefore, it holds that $\theta_z=\sum_{i=1}^{k} a_i v_i\in Cone_{\sum_{i=1}^{k}a_it_i}$. Consider this $\theta_0=\sum_{i=1}^{k}a_it_i$, because $Q_0$ is convex, we have $\theta_0\in Q_0$. In addition,
		$\|\theta_0-\theta_z\|_2\leq\frac{|z|}{\|d\|_2\cdot\sin(\beta_{\min})}$,  which means by choosing $C_L=\frac{1}{\|d\|_2\cdot\sin(\beta_{\min})}$, the proof is complete.
	\end{proof}

	Denote $z_{\delta}:=\tilde{\theta}_{\delta}$ and we know that $|z_{\delta}|\leq\delta$. In order to apply \Cref{lemma:bounded_shifting}, we have to ensure that $\tilde{\cD}_{abg}(0)\neq\emptyset$. In fact, notice that $\tilde{\theta}_0\in\tilde{\cD}_{abg}\cap\{\mathbf{d}^{\top}=0\}$. With \Cref{lemma:bounded_shifting}, there exists a $\hat{\theta}_0\in\tilde{\cD}_{abg}(0)$ such that $\|\tilde{\theta}_{\delta}-\hat{\theta}_0\|_2\leq\frac{L}2|z|\leq\frac{L}2\delta.$ As a result, we have:
	\begin{equation}
	\label{equ:bounded_tilde_with_hat_close}
	\begin{aligned}
	\mathbf{a}^{\top}\tilde{\theta}_{\delta}-\mathbf{a}^{\top}\hat\theta_0&\leq\|a\|_2\cdot\|\tilde{\theta}_{\delta}-\hat{\theta}_0\|_2\\
	&\leq\|a\|_1\cdot C_L\cdot\delta\\
	&\leq(q\cdot\vv^{\top}F_1\ind+(1-q)\cdot\vv^{\top}F_2\ind)\cdot C_L\cdot\delta\\
	&:= C_a\cdot C_L\cdot\delta.
	\end{aligned}
	\end{equation}
	By definition of $\tilde{\theta}_{\delta}$, we know that $\tilde{\theta}_0$ maximizes $\mathbf{a}^{\top}\theta$ in $\tilde{\cD}_{abg}(0)$, which means that $\mathbf{a}^{\top}\tilde{\theta}_0\geq\mathbf{a}^{\top}\hat{\theta}_0$. As a result, we have:
	
	\begin{equation}
	\label{equ:bound_performance_within_delta_gap}
	\begin{aligned}
	R(\theta_{\delta})-R(\theta_0)=&\mathbf{a}^{\top}\tilde{\theta}_{\delta}-\mathbf{a}^{\top}\tilde{\theta}_0\\
	\leq&\mathbf{a}^{\top}\tilde{\theta}_{\delta}-\mathbf{a}^{\top}\hat{\theta}_0\\
	\leq&C_a\cdot C_L\cdot\delta.
	\end{aligned}
	\end{equation}
	Therefore, we have $R(\pi_{\delta, *})-R(\pi_*)\leq R(\theta_{\delta})-R(\pi_{start})=R(\theta_{\delta})-R(\theta_0)\leq C_a\cdot C_L\cdot\delta$. Let $L:= 2\cdot C_a\cdot C_L$ and this holds the lemma.
\end{proof}
\subsection{Proof of \Cref{theorem:regret_lower_bound}}
\label{appendix:proof_regret_lower_bound}
As is stated in \Cref{subsec:regret_lower_bound}, we may reduce this fair pricing problem to an ordinary online pricing problem with no fairness constraints. Therefore, we only need to prove the following theorem.
\begin{theorem}[Regret Lower Bound]
	Consider the online pricing problem with $T$ rounds and $d$ fixed prices in $[0,c]$ for  $3\leq d\leq T^{1/3}$ and some constant $c>0$. Then any algorithm has to suffer at least $\Omega(\sqrt{dT})$ regret.
\end{theorem}

Here we mainly adopt the proof roadmap of \citet{kleinberg2003value}.
\begin{proof}
	We let $c=12$ without losing generality. Let $\epsilon = \sqrt{\frac{d}{T}}, l=1, a_0=4l, a_i = (1+\frac{\epsilon}{l})^{i}\cdot a_0, i=1,2,\ldots, d$, then we have: $4l=a_0<a_1<a_2<\ldots<a_{d-1}<a_{d}<12l$.
	
	Define some distributions on the prices $\{a_i\}_{i=1}^{d}$:
	\begin{itemize}
		\item $\P_0$, with acceptance rates of each price: $P_0=[\frac{l}{a_1}, \frac{l}{a_2}, \ldots, \frac{l}{a_{d-1}}, \frac{l}{a_{d}}]^T$, where $P_0(i)=\Pr[y\geq a_i] = \Pr[y>a_{i-1}]=\frac{l}{a_i}<\frac{1}{4}$.
		\item $\P_j$, with  acceptance rates of each price: $P_j=[\frac{l}{a_1}, \frac{l}{a_2}, \ldots, \frac{l}{a_{j-1}}, \frac{l+\epsilon}{a_j}, \frac{l}{a_{j+1}}, \ldots, \frac{l}{a_{d-1}}, \frac{l}{a_{d}}]^T$, where $P_{j}(i)=\frac{l}{a_{i}}+\frac{\epsilon}{a_i}\cdot\ind(i=j)\leq\frac{1}{4}$.
	\end{itemize}
	In the following part, we propose and prove the following lemma: 
	
	\begin{lemma}
		For any algorithm $S$, $\exists j\in\{1,2,\ldots, d\}$, such that $Reg_{\P_j}(S)=\Omega(\sqrt{Td})$.
	\end{lemma}

	\begin{proof}
		Suppose $f$ is a function: $\{0,1\}^{T}\rightarrow[0,M]$. Denote $\vr=[\ind_1, \ind_2, \ldots, \ind_T]^{\top}$ as a vector containing the customer's decisions in sequence. Then for any $j=1,2,\ldots, d$ we have:
		\begin{equation}\label{equ:information_chain_rule}
		\begin{aligned}
		&\E_{\P_j}[f(\vr)]-\E_{\P_0}[f(\vr)]\\
		=&\sum_{\vr}f(\vr)\cdot(\P_j[\vr]-\P_0[\vr])\\
		\leq&\sum_{\vr:\P_{j}[\vr]\geq\P_0[\vr]}f(\vr)(\P_{j}[\vr]-\P_{0}[\vr])\\
		\leq&M\cdot\sum_{\vr:\P_{j}[\vr]\geq\P_0[\vr]}f(\vr)(\P_{j}[\vr]-\P_0[\vr])\\
		=&\frac{M}{2}\|\P_{j}-\P_0\|_1.
		\end{aligned}
		\end{equation}
		Here we cite a lemma from \textit{Cover and Thomas, Elements of Information theory}, Lemma 11.6.1.
		\begin{lemma}
			\begin{equation*}
			KL(\P_1||\P_2)\geq\frac{1}{2\ln{2}}\|\P_1-\P_2\|_1^2.
			\end{equation*}
			\label{lemma_11_6_1}
		\end{lemma}
		Since
		
		\begin{equation}\label{equ:kl_decomposition}
		\begin{aligned}
		&KL(\P_0(\vr)||\P_{j}(\vr))\\
		=&\sum_{t=1}^{T}KL(\P_0[r_t|\vr_{t-1}]||\P_j[r_t|\vr_{t-1}])\\
		=&\sum_{t=1}^{t}\P_0(i_t\neq{j})\cdot 0 + \P_0(i_t = j)\cdot KL(\frac{l}{a_j}||\frac{l}{a_j}+\frac{\epsilon}{a_j}).
		\end{aligned}
		\end{equation}
		The first equality comes from the chain rule of decomposing a KL-divergence, and the second equality is because $\ind_t$ satisfies a Bernoulli distribution $B(1, \frac{l}{a_{i_t}}+\frac{\epsilon}{a_{i_t}}\cdot\ind(i_t=j))$. Now we propose another lemma:
		\begin{lemma}
			If $\frac{1}{12}\leq{p}\leq\frac14,\text{ then we have: } KL(p, p+\epsilon)\leq12\epsilon^2$ for sufficiently small $\epsilon$.
			\label{lemma_kl_ber}
		\end{lemma}
		According to this Lemma \ref{lemma_kl_ber}, we have:
		\begin{equation}
		\begin{aligned}
		KL(\P_0(\vr)||\P_{j}(\vr))&\leq\sum_{t=1}^{T}\P_0(i_t=j)\cdot12\epsilon^2\\
		&\leq\sum_{t=1}^T\P_0(i_t=j)\cdot\frac{12}{16l^2}\epsilon^2.
		\end{aligned}
		\end{equation}
		Therefore, we have:
		\begin{equation}
		\begin{aligned}
		&\E_{\P_j}[f(\vr)]-\E_{\P_0}[f(\vr)]\\
		\leq&\frac{M}{2}\frac{2\ln{2}}\cdot\sqrt{KL(\P_0(\vr)||\P_{j}(\vr))}\\
		\leq&\frac{\sqrt{6\ln{2}}M}{4}\cdot(\sqrt{\sum_{t=1}^T\P_0(i_t=j)})\cdot\epsilon.\\
		\end{aligned}
		\end{equation}
		Denote $N_j:=\sum_{t=1}^T\ind(i_t=j)$, and hence:
		\begin{equation}
		\E_{\P_j}[f(\vr)]-\E_{\P_j}[f(\vr)]\leq\frac{\sqrt{6\ln{2}}}{4}M(\sqrt{\E_{\P_0}[N_j]})\cdot\epsilon.
		\end{equation}
		Now let $f(\vr)=N_j$, i.e., let function $f$ simulate the algorithm which make choices of $i_t$'s from historical results of $\{\ind_1, \ind_2, \ldots, \int_{t-1}\}$ (It is straightforward that $\{i_1, i_2, \ldots, i_{t-1}\}$ are also historical results crucial for deciding $i_t$. However, for a deterministic algorithm, it can generate $i_1, i_2, \ldots, i_{t-1}$ directly from $\emptyset, \{\ind_1\}, \{\ind_1, \ind_2\}, \ldots, \{\ind_1, \ind_2, \ldots, \ind_{t-2}\}$.) Now, $0\leq{f}(\vr)\leq{T}\text{, which indicates that } M=T\text{.}$. Then it turns out that
		\begin{equation}
		\begin{aligned}
		\E_{\P_j}[N_j]-\E_{\P_0}[N_j]&\leq\frac{\sqrt{6\ln{2}}}{4}\cdot T \cdot \epsilon \cdot \sqrt{\E_{\P_0}[N_j]}\\
		\Rightarrow\frac{1}{d}\sum_{j=1}^{d}\E_{\P_j}[N_j]&\leq\frac{1}{d}\sum_{j=1}^{d}(\E_{\P_0}[N_j]+\frac{\sqrt{6\ln{2}}}{4}\cdot T\cdot\epsilon\cdot \sqrt{\E_{\P_0}[N_j]})\\
		&=\frac{T}{d}+\frac{\sqrt{6\ln{2}}}{4}\cdot\frac{T}{d}\cdot\epsilon\cdot\sum_{j=1}^d\sqrt{\E_{\P_)}[N_j]}\\
		&\leq\frac{T}{d}+\frac{\sqrt{6\ln{2}}}{4}\cdot\epsilon\cdot\frac{T}{d}\cdot\sqrt{Td}\\
		&=\frac{T}{d} + \frac{\sqrt{6\ln2}}{4}\cdot\sqrt{\frac{d}{T}}\cdot\frac{T}{d}\cdot\sqrt{Td}\\
		&\leq\frac{T}3 + 0.525\cdot T\\
		&\leq0.9T
		\end{aligned}
		\label{equation_average}
		\end{equation}
		Here the second line is an average over all $j=1,2,\ldots, d$ of the first line, the third line uses the fact that $\sum_{j=1}^d\E[N_j] = T$, the fourth line applies a Cauchy-Schwarz Inequality that $Td=(\sum_{j=1}^{d}\E_{\P_0}[N_j])(d\cdot 1)\geq(\sum_{j=1}^{d}\sqrt{\E_{\P_0}[N_j]})^2$, the fifth line plugs in the values that $ \epsilon =\sqrt{\frac{d}{T}}$, the sixth line uses the fact that $\ln{2}<0.7$, and the last line holds for sufficient large $T$.
		
		From Equation \ref{equation_average}, we know that $\exists j\in\{1,2,\ldots,d\}$ such that $\E_{\P_j}[N_j]\leq0.9T$. As a result, we have:
		\begin{equation}
		\begin{aligned}
		Reg_{\P_j}(S)\geq&(1-0.9)T(\frac{l+\epsilon}{a_j}\cdot{a_j}-\frac{l}{a_{i_t}}\cdot a_{i_t})\text{, $\forall i_t\neq{j}$}\\
		=&0.1T(l+\epsilon-l)\\
		=&0.1T\epsilon\\
		=&0.1\sqrt{Td}.
		\end{aligned}
		\end{equation}
		Therefore,  the $\Omega(\sqrt{Td})$ regret bound holds.
	\end{proof}

\end{proof}
	
\subsection{Proof of \Cref{theorem:unfairness_lower_bound}}
\label{appendix:proof_unfairness_lower_bound}
\begin{proof}
	Prior to our technical analysis, we briefly introduce the roadmap of proving the unfairness lower bound.
	\begin{enumerate}[label = (\roman*)]
		\item We construct two different but very similar problem settings: one is exactly \Cref{example:random_policy}, the other is identical to \Cref{example:random_policy} except all probabilities of $0.5$ are now changed into $(0.5-\ep)$, where $\ep = C\cdot T^{-\frac12+\eta}$ for some super small constant $C\geq0$ and some small $\eta\geq0$. In the following, we may call them the ``Problem $0$'' (or $P_0$) and the ``Problem $\ep$'' (or $P_{\ep}$) sequentially.
		\item We derive the close-form solutions to both Problem $0$ and Problem $\ep$, where we also parameterize the reward function with the expected proposed price $V_r$ and the proposed accepted price $V_s$. Of course Problem $\ep$ is more general and we may get the solutions to Problem $0$ by simply let $\ep=0$.
		\item We show that there does not exist any policy $\pi$ that satisfies both of the following conditions simultaneously:
		\begin{itemize}
			\item $\pi$ is within $C_0\cdot T^{-\frac12+\eta}$-suboptimal (w.r.t. regret) and within $C_0\cdot T^{-\frac12+\eta}$-unfair (w.r.t. fairness) in $P0$.
			\item $\pi$ is within $C_0\cdot T^{-\frac12+\eta}$-suboptimal (w.r.t. regret) and within $C_0\cdot T^{-\frac12+\eta}$-unfair (w.r.t. fairness) in $P\ep$.
		\end{itemize}
		\item We show that any algorithm have to distinguish  $P_0$
	\end{enumerate}

	According to the roadmap above, we firstly construct the following example as the problem setting for lower bound proof.
	\begin{example}\label{example:lowerbound_construction}
		Customers form 2 disjoint groups: Group 1 takes 30\% proportion of customers, and Group 2 takes the rest 70\%. In specific,
		\begin{itemize}
			\item In Group 1, $40\%$ customers valuate the item as \$0, $(10\%+\ep)$ valuate customers it as \$0.625, and $(50\%-\ep)$ customers valuate it as \$1.
			\item In Group 2, $20\%$ customers valuate it as \$0, $(30\%+\ep)$ customers valuate it as \$0.7, and $(50\%-\ep)$ customers valuate it as \$1.
		\end{itemize}
		Here $\ep=C\cdot T^{-\frac12+\eta}$ is a small amount, where $0\leq C\leq 10^{-10}$. In other words, we have $\vv^{\top}=[\frac58, \frac7{10}, 1], F_1 =\diag\{0.6, 0.5-\ep, 0.5-\ep\}, F_2 = \diag\{0.8, 0.8, 0.5-\ep\}$ and our policy $\pi = (\pi^1, \pi^2)$ where $\pi^1, \pi^2\in\Delta^3$. Our goal is to approach the following optimal policy
		\begin{equation}
		\label{equ:pi_ep_star}
		\begin{aligned}
		\pi_{\ep,*} =& \argmax_{\pi=(\pi^1, \pi^2)\in\Pi} R(\pi; F_1, F_2)\\
		s.t.&\quad U(\pi) = 0\\
		&\quad S(\pi; F_1, F_2) = 0.\\
		\end{aligned}
		\end{equation}
	\end{example}
	For any policy $\pi$ feasible to the constraints in \Cref{equ:close_form}, denote its expected accepted price as $V_s$ (identical in both groups) and its expected proposed price as $V_r$ (identical in both groups as well). Notice that $V_r\geq V_s\geq\frac58$, we define $\alpha=V_r-V_s$ as their difference, and therefore we know that $\alpha\geq0$. Again, we denote $R(\pi, F_1, F_2)$ as $R(\pi)$ without causing misunderstandings. Here we propose the following lemma regarding \Cref{example:lowerbound_construction}.
	\begin{lemma}[Close-form solution to \Cref{example:lowerbound_construction}]
		For the problem setting defined in \Cref{example:lowerbound_construction}, we have:
		\begin{equation}
		\label{equ:close_form}
		\begin{aligned}
		\pi_{\ep,*}^1 = &[\frac{20-40\ep}{29-10\ep}, 0, \frac{9+30\ep}{29-10\ep}]^{\top},\\
		\pi_{\ep,*}^2 = &[0, \frac{25-50\ep}{29-10\ep}, \frac{4+40\ep}{29-10\ep}]^{\top}, \forall\ep\in[0,10^{-10}].
		\end{aligned}
		\end{equation}
		Besides, for any feasible policy $\pi$ and its corresponding $V_s$ and $\alpha$, we have:
		\begin{equation}
		\label{equ:reward_function_close_form}
		R(\pi) = \frac{71-30\ep}{100}\cdot V_s + \frac{(100-60\ep)-(142-60\ep)V_s}{(8V_s-5)(1-V_s)25}\cdot V_s\cdot \alpha.
		\end{equation}
		\label{lemma:close_form_solution_with_epsilon}
	\end{lemma}
	\begin{proof}[Proof of \Cref{lemma:close_form_solution_with_epsilon}]
		For any feasible policy $\pi=(\pi^1, \pi^2)$, it has to satisfy the following equations for $e=1,2$:
		\begin{equation*}
		\left\{
		\begin{aligned}
		&\ind^{\top}\pi^e=1\\
		&\vv^{\top}\pi^e=V_s + \alpha\\
		&\frac{\vv^{\top}F_e\pi^e}{\ind^{\top}F_e\pi^e} = V_s.
		\end{aligned}
		\right.
		\end{equation*}
		This is equivalent to the following linear equations system
		\begin{equation}
		\label{equ:feasibility_constraints}
		\left\{
		\begin{aligned}
		&\ind^{\top}\pi^e=1\\
		&\vv^{\top}\pi^e=V_s + \alpha\\
		&(\vv-V_s\cdot\ind)^{\top}F_e\pi^e = 0.
		\end{aligned}
		\right.
		\end{equation}
		This is further equivalent to $A_1\pi^1 = [1,V_s+\alpha, 0]^{\top}$ and $A_2\pi^2=[1,V_s+\alpha, 0]^{\top}$ where
		\begin{equation}
		\label{equ:A1_matrix}
		A_1(V_s,\ep) = 
		\begin{bmatrix}
		1 & 1 & 1 \\
		\frac58 & \frac7{10} & 1 \\
		(\frac58-V_s)\cdot\frac35 & (\frac7{10}-V_s)\cdot(\frac12-\ep) & (1-V_s)\cdot(\frac12-\ep)
		\end{bmatrix}
		.
		\end{equation}
		and
		\begin{equation}
		\label{equ:A2_matrix}
		A_2(V_s,\ep) = 
		\begin{bmatrix}
		1 & 1 & 1 \\
		\frac58 & \frac7{10} & 1 \\
		(\frac58 - V_s)\cdot\frac45 & (\frac7{10}-V_s)\cdot\frac45 & (1-V_s)\cdot(\frac12-\ep)\\
		\end{bmatrix}.
		\end{equation}
		Here we may omit the parameters $(V_s, \ep)$ without misunderstanding. For $V_s = \frac58$, the only possible policy is to propose the lowest price $\frac58$ for both groups, and the expected reward is $0.3\times\frac58\times\frac35 + 0.7\times\frac58\times\frac45 = 0.4625<0.5-\ep$. Therefore, it is suboptimal as its expected reward is less than that of a deterministic policy keep proposing $1$ as a price (whose reward is $0.5-\ep$). In the following, we only consider the case when $V_s>\frac58$. Solve these linear equation systems and get
		\begin{equation}
		\label{equ:pi_1_close_form}
		\begin{aligned}
		\pi^1 =& A_1^{-1}[1,V_s+\alpha, 0]^{\top}\\
		=&\frac1{3(8V_s-5)(1+10\ep)}\\
		&\begin{bmatrix}
		120\alpha(1-2\ep)\\
		-((1+10\ep)8V_s + 10(1-8\ep))\cdot\alpha - (8(1+10\ep)V_s^2-13(1+10\ep)V_s + (1+10\ep)5)\\
		10(8(1+10\ep)V_s - 2(1+28\ep))\cdot\alpha + (1+10\ep)(8V_s-5)(10V_s-7) 
		\end{bmatrix}
		\end{aligned}
		\end{equation}
		and
		\begin{equation}
		\label{equ:pi_2_close_form}
		\begin{aligned}
		\pi^2 =& A_2^{-1}[1, V_s+\alpha, 0]^{\top}\\
		 =& \frac1{3(1-V_s)(3+10\ep)}\\
		&\begin{bmatrix}
		4(((3+10\ep)10V_s - (6+100\ep))\alpha-(3+10\ep)(10V_s-7)(V_s-1))\\
		(-5)\cdot(((3+10\ep)8V_s-80\ep)\alpha+(3+10\ep)(8V_s-5)(V_s-1))\\
		24\alpha
		\end{bmatrix}.
		\end{aligned}
		\end{equation}
		On the one hand, we can get the explicit form of $R(\pi)$ w.r.t. $V_s$ and $\alpha$:
		\begin{equation}
		\label{equ:r_explicit}
		\begin{aligned}
		R(\pi)=&q\cdot\vv^{\top}F_1\pi^1 + (1-q)\cdot\vv^{\top}F_2\pi^2\\
		=&\frac{71-30\ep}{100}V_s+\frac{(100-60\ep)-(142-60\ep)V_s}{(8V_s-5)(1-V_s)25}\cdot V_s\alpha.
		\end{aligned}
		\end{equation}
		On the other hand, we have a few constraints to be applied. Since $\pi$ is a probabilistic distribution, we have $\pi^e(i)\geq0, e=1,2; i=1,2,3$, which lead to
		\begin{equation}
		\label{equ:constraints_piei_positive}
		\left\{
		\begin{aligned}
		&120\alpha(1-2\ep)\geq 0\\
		&-((1+10\ep)8V_s + 10(1-8\ep))\cdot\alpha - (8(1+10\ep)V_s^2-13(1+10\ep)V_s + (1+10\ep)5)\geq 0\\
		&10(8(1+10\ep)V_s - 2(1+28\ep))\cdot\alpha + (1+10\ep)(8V_s-5)(10V_s-7)\geq 0\\
		&4(((3+10\ep)10V_s - (6+100\ep))\alpha-(3+10\ep)(10V_s-7)(V_s-1))\geq 0\\
		&(-5)\cdot(((3+10\ep)8V_s-80\ep)\alpha+(3+10\ep)(8V_s-5)(V_s-1))\geq 0\\
		&24\alpha\geq 0.
		\end{aligned}
		\right.
		\end{equation}
		From \Cref{equ:constraints_piei_positive}, we may derive the following upper and lower bounds for $\alpha$.
		\begin{enumerate}[label=(\alph*)]
			\item The first line and the last line of \Cref{equ:constraints_piei_positive} is naturally satisfied.
			\item From the second line, we have
			\begin{equation}
			\label{equ:alpha_bound_one}
			\alpha\leq\frac{(1+10\ep)(8V_s-5)(1-V_s)}{(1+10\ep)8V_s+10(1-8\ep)}:=B_1.
			\end{equation}
			\item From the third line, we have
			\begin{equation}
			\label{equ:alpha_bound_two}
			\alpha\geq\frac{(1+10\ep)(8\cdot V_s-5)(7-10V_s)}{(1+10\ep)8\cdot V_s-2(1+28\ep)}\cdot\frac1{10}:=B_2.
			\end{equation}
			\item From the fourth line, we have
			\begin{equation}
			\label{equ:alpha_bound_three}
			\alpha\geq\frac{(3+10\ep)(10V_s-7)(1-V_s)}{(3+10\ep)10V_s-(6+100\ep)}:=B_3.
			\end{equation}
			\item From the fifth line, we have
			\begin{equation}
			\label{equ:alpha_bound_four}
			\alpha\leq\frac{(3+10\ep)(8V_s-5)(1-V_s)}{(3+10\ep)8\cdot V_s-80\ep}:=B_4.
			\end{equation}
		\end{enumerate}
	We get four constraints on $\alpha$ as above, where \Cref{equ:alpha_bound_one} and \Cref{equ:alpha_bound_four} are upper bounds, and \Cref{equ:alpha_bound_two} and \Cref{equ:alpha_bound_three} are lower bounds. Compare $B_1$ with $B_4$, we notice that
	\begin{equation}
	\begin{aligned}
	\frac{B_1}{B_4}=&\frac{\frac{80\ep}{3+10\ep}}{8V_s + 10\cdot\frac{1-8\ep}{1+10\ep}}\leq  1.
	\end{aligned}
	\end{equation}
	Therefore, \Cref{equ:alpha_bound_one} is tighter than \Cref{equ:alpha_bound_four}. For the comparison between $B_2$ and $B_3$, we notice that $B_2<0<B_3$ when $V_s>\frac7{10}$ and $B_2\geq0\geq B_3$ when $V_s\leq\frac7{10}$.
	
	In the following part, we derive the optimal policy by cases.
	\begin{enumerate}[label = (\alph*)]
		\item When $\frac58<V_s\leq\frac{50-30\ep}{71-30\ep}$, we have
		\begin{equation}
		\label{equ:when_v_s_is_rather_smallerthan_5071}
		\begin{aligned}
		R(\pi)=&\frac{71-30\ep}{100}V_s+\frac{(100-60\ep)-(142-60\ep)V_s}{(8V_s-5)(1-V_s)25}\cdot V_s\alpha\\
		\leq&\frac{71-30\ep}{100}V_s+\frac{(100-60\ep)-(142-60\ep)V_s}{(8V_s-5)(1-V_s)25}\cdot V_s\cdot B_1\\
		=&\frac{71-30\ep}{100}V_s+\frac{(100-60\ep)-(142-60\ep)V_s}{(8V_s-5)(1-V_s)25}\cdot V_s\cdot\frac{(1+10\ep)(8V_s-5)(1-V_s)}{(1+10\ep)8V_s+10(1-8\ep)}\\
		=&\frac{71-30\ep}{100}V_S+\frac{100-142V_s-60\ep(1-V_s)}{25(8V_s+10\cdot\frac{1-8\ep}{1+10\ep})}\cdot V_s\\
		=&\frac{71-30\ep}{100}V_S+\frac{100-142V_s-60\ep(1-V_s)}{25(8V_s+10)-\frac{450\ep}{1+10\ep}}\cdot V_s\\
		<&\frac{71-30\ep}{100}V_S+\frac{100-142V_s-60\ep(1-V_s)+\frac{450\ep}{1+10\ep}}{25(8V_s+10)}\cdot V_s\\
		<&\frac{71}{100}V_s + \frac{100.1-142V_s}{25(8V_s+10)}\cdot V_s\\
		=&\frac{71(8\cdot V_s+10)+(100.1-142\cdot V_s)\cdot4}{100(8\cdot V_s + 10)}\cdot V_s\\
		=&\frac{11104\cdot V_s}{1000(8\cdot V_s + 10)}\\
		\leq&\frac{11104}{8000}-\frac{\frac54\times11104}{1000(8\times\frac{50}{71}+10)}\\
		\leq &0.50019
		\end{aligned}
		\end{equation}
		Here the first inequality (line 2) is by $(100-60\ep)-(142-60\ep)V_s\geq0$ as $V_s\leq\frac{50-30\ep}{71-30\ep}$ and by $\alpha\leq B_1$, the second inequality (line 6) is by the fact that $V_s$'s coefficient is within $(0,1)$, the third inequality (line 7) is by $\ep\leq10^{-10}<\frac1{4500}$, the fourth inequality (line 10) is by $V_s\leq\frac{50}{71}$ and the last inequality (line 11) is by numerical computations. We will later show that 0.50019 is not optimal.
		\item When $V_s>\frac{50-30\ep}{71-30\ep}$, we know that $V_s>\frac7{10}$ as $\ep<\frac1{30}$. Therefore, we know $B_2<0<B_3$ and we have
		\begin{equation*}
		\alpha\geq B_3 = \frac{(3+10\ep)(10V_s-7)(1-V_s)}{(3+10\ep)10V_s-(6+100\ep)}.
		\end{equation*}
		As a result, we have
		\begin{equation}
		\label{equ:alpha_equal_b3}
		\begin{aligned}
		R(\pi)&=\frac{71-30\ep}{100}V_s-\frac{(142-60\ep)V_s-(100-60\ep)}{(8V_s-5)(1-V_s)25}\cdot V_s\alpha\\
		&\leq\frac{71-30\ep}{100}V_s-\frac{(142-60\ep)V_s-(100-60\ep)}{(8V_s-5)(1-V_s)25}\cdot V_s\cdot B_3\\
		&\leq\frac{71-30\ep}{100}V_s-\frac{(142-60\ep)V_s-(100-60\ep)}{(8V_s-5)(1-V_s)25}\cdot V_s\cdot\frac{(3+10\ep)(10V_s-7)(1-V_s)}{(3+10\ep)10V_s-(6+100\ep)}\\
		\end{aligned}
		\end{equation}
		Also, we can derive an upper bound for $V_s$ as $B_3\leq\alpha\leq B_1$.
		\begin{equation}
		\label{equ:b3_leq_b1}
		\begin{aligned}
		\frac{(3+10\ep)(10V_s-7)(1-V_s)}{(3+10\ep)10V_s-(6+100)}\leq&\frac{(1+10\ep)(8\cdot V_s-5)(1-V_s)}{(1+10\ep)8\cdot V_s+10(1-8\ep)}\\
		\Leftrightarrow\qquad V_s\leq&\frac{8+10\ep}{11+10\ep}.
		\end{aligned}
		\end{equation}
		Therefore, we may solve the maximal of $R(\pi)$ on $\frac{50-30\ep}{71-30\ep}\leq V_s\leq\frac{8+10\ep}{11+10\ep}$ by combining \Cref{equ:alpha_equal_b3}.
		\begin{equation}
		\label{equ:when_v_s_is_large_enough}
		\begin{aligned}
		\frac{\partial R(\pi)}{\partial V_s}&=\frac{3 (-3135 + 9870 V_s - 7491 V_s^2) + 3 \ep (-46430 + 145660 V_s - 114638 V_s^2)}{20 (-5 + 8 V_s)^2 (-3 - 50 \ep + 15 V_s + 50 \ep V_s)^2}\\
		&+\frac{ + 3 \ep^2 (97900 - 315800 V_s + 251740 V_s^2) + 3 \ep^3 (15000 - 30000 V_s + 15000 V_s^2)}{20 (-5 + 8 V_s)^2 (-3 - 50 \ep + 15 V_s + 50 \ep V_s)^2}\\
		&=\frac{-22473(V_s-\frac{1645-6\sqrt{2685}}{2497})(V_s-\frac{1645+6\sqrt{2685}}{2497})+3(-46430 + 145660 V_s - 114638 V_s^2)\ep + o(\ep^2)}{20 (-5 + 8 V_s)^2 (-3 - 50 \ep + 15 V_s + 50 \ep V_s)^2}\\
		&\gtrsim\frac{-22473(V_s-0.5343)(v_s-0.7833)}{20 (-5 + 8 V_s)^2 (-3 - 50 \ep + 15 V_s + 50 \ep V_s)^2}\\
		&>0.
		\end{aligned}
		\end{equation}
		Here the ``$\gtrsim$'' inequality is because the coefficient of $\ep$ in any monomial above is within $\pm10^6$, which indicates that any monomial containing $\ep$ is within $\pm0.0001$ . The last line is because $\frac{7}{10}\leq\frac{50-30\ep}{71-30\ep}\leq V_s\leq\frac{8+10\ep}{11+10\ep}\leq\frac{3}{4}$ and therefore $(V_s-0.5343)(v_s-0.7833)<0$. As a result, we know that $R(\pi)$ is monotonically increasing as $V_s$ increases within the range above. Therefore, we have:
		\begin{equation}
		\label{equ:optimal_r}
		\begin{aligned}
		R(\pi)\leq R(\pi)|_{\alpha=B_3}\leq R(\pi)|_{\alpha=B_3\text{ and } V_s = \frac{8+10\ep}{11+10\ep}}=\frac{37(1-2\ep)(4+5\ep)}{10(29-10\ep)}
		\end{aligned}
		\end{equation}
		By plugging in $V_s=\frac{8+10\ep}{11+10\ep}$ into $\alpha=B_3$ and the close-form feasible solutions of $\pi^1$ and $\pi^2$ (i.e., \Cref{equ:pi_1_close_form} and \Cref{equ:pi_2_close_form} ), we may get:
		\begin{equation}
		\label{equ:push_back}
		\begin{aligned}
		\alpha=B_3&=\frac{3(1+10\ep)(3+10\ep)}{2(29-10\ep)(11+10\ep)}\\
		\pi^1&=[\frac{20-40\ep}{29-10\ep}, 0, \frac{9+30\ep}{29-10\ep}]^{\top}\\
		\pi^2&=[0,\frac{25-50\ep}{29-10\ep}, \frac{4+40\ep}{29-10\ep}]^{\top}.
		\end{aligned}
		\end{equation}
		
		Pushing back \Cref{equ:push_back} to \Cref{equ:pi_ep_star}, we verify that $R(\pi)_{\max} = \frac{37(1-2\ep)(4+5\ep)}{10(29-10\ep)}$ and therefore all inequalities in \Cref{equ:optimal_r} hold as equalities.
	\end{enumerate}  
	Notice that $\frac{37(1-2\ep)(4+5\ep)}{10(29-10\ep)}>0.50019$, and therefore the optimal policy $\pi_{\ep, *}$ is what we derive in \Cref{equ:push_back}. This holds the lemma.
	\end{proof}
	With \Cref{lemma:close_form_solution_with_epsilon}, we know that $\pi_{\ep,*}^1=[\frac{20-40\ep}{29-10\ep}, 0, \frac{9+30\ep}{29-10\ep}]^{\top} $ and $ \pi_{\ep,*}^2=[0,\frac{25-50\ep}{29-10\ep}, \frac{4+40\ep}{29-10\ep}]^{\top} $. We denote $V_{s,\ep}^*:=\frac{8+10\ep}{3+10\ep}$ and $\alpha_{\ep}^*=\frac{3(1+10\ep)(3+10\ep)}{2(29-10\ep)(11+10\ep)}$ for future use. We also know that the optimal policy for \Cref{example:random_policy} (i.e., $\ep = 0$) is exactly what we proposed, i.e., $\pi_*^1=[\frac{20}{29}, 0, \frac{9}{29}]^{\top}$ and $\pi_*^2=[0, \frac{25}{29},\frac{4}{29}]^{\top}$.
	
	Let us go back to the two problems: $P_0$ defined in \Cref{example:random_policy} and $P_{\ep}$ defined in \Cref{example:lowerbound_construction}, where we consider the following four conditions:
	\begin{itemize}
		\item $\pi$ is within $C_0\cdot T^{-\frac12+\eta}$-suboptimal (w.r.t. regret) in $P_0$ (denoted as Condition A).
		\item $\pi$ is within $C_0\cdot T^{-\frac12+\eta}$-suboptimal (w.r.t. regret) in $P_{\ep}$ (denoted as Condition B).
		\item $\pi$ is within $C_0\cdot T^{-\frac12+\eta}$-unfair (w.r.t. fairness) in $P_0$ (denoted as Condition C).
		\item $\pi$ is within $C_0\cdot T^{-\frac12+\eta}$-unfair (w.r.t. fairness) in $P_{\ep}$ (denoted as Condition D).
	\end{itemize}
	According to our proof roadmap, we then prove the following lemma:
	\begin{lemma}[No policy fitting in $P_0$ and $P_{\ep}$]
	There exist constants $C_0>0$ 
	such that there does not exist any policy $\pi\in\Pi$ that satisfies all of Condition $ABCD$ (denoting $A\wedge B\wedge C\wedge D$) simultaneously.
	\label{lemma:no_perfect_policy}
	\end{lemma}
	\begin{corollary}\label{corollary:3-parity}
		The space of $\Pi$ can be divided as the following 3 subspaces:
		\begin{enumerate}
			\item Policies satisfying Condition $AC$ (denoted as Space AC).
			\item Policies satisfying Condition $BD$ (denoted as Space BD).
			\item Policies satisfying Condition (denoted as Outer Spaces) $\bar{A}\bar{B}\vee\bar{C}\bar{D}\vee\bar{A}\bar{D}\vee\bar{B}\bar{C}$.
		and these three subspaces are pairwise disjoint.
		\end{enumerate}
	\end{corollary}
	\begin{proof}[Proof of \Cref{lemma:no_perfect_policy}]
		Let $C_1=\frac{C}{W}$ and $C_2=\frac{C}{W\cdot L}$ where $L>0$ is a constant from \Cref{lemma:small_relaxation_gain} and $W\geq 10$ to be specified later. Let $C_0 = \min\{C_1, C_2\}$, and we prove the lemma by contradiction. Suppose there exists a policy $\pi$ satisfies the four conditions above, and then we denote the expected accepted prices in $G_1$ and $G_2$ in Problem $\ep$ are $V_{s,\ep}$ and $V_{s, \ep}+\beta_{\ep}$ sequentially, where $\beta_{\ep}\in[0,C_2 T^{-\frac12+\eta}]$. Here we assume $\beta\geq0$ without losing generality as we will not use the specific property of $G_1$ versus $G_2$. Also, we denote $\alpha_{\ep}$ as the difference between the expected proposed price in both groups (denoted as $V_{r, \ep}$) and $V_{s, \ep}$.
		
		Now, consider a corresponding policy:
		\begin{equation}
		\label{equ:check_pi}
		\check{\pi}:=\left\{
		\begin{aligned}
		G_1&:\E[\text{accepted price}]=V_{s,\ep}, \E[\text{proposed price}]=V_{s,\ep}+\alpha_{\ep}\\
		G_2&:\E[\text{accepted price}]=V_{s,\ep}, \E[\text{proposed price}]=V_{s,\ep}+\alpha_{\ep}\\
		\end{aligned}
		\right.
		\end{equation}
		According to \Cref{equ:b3_leq_b1}, we know that $\frac58\leq V_{s,\ep}\leq V_{s,\ep}^*=\frac{8+10\ep}{11+10\ep}$ and $R(\check{\pi})\leq R(\pi_{\ep,*})$. Therefore, we have:
		\begin{equation}
		\label{equ:performance_difference_pi_pispstar}
		\begin{aligned}
		R(\pi)&\leq R(\check{\pi}) + L\cdot\beta_{\ep}\\
		&\leq R(\pi_{\ep, *})-\min_{V_s\in[\frac58, \frac{8+10\ep}{11+10\ep}]}\frac{\partial R(\pi)}{\partial V_s}\cdot (V_{s,\ep}^*-V_{s,\ep}) + L\cdot\beta_{\ep}\\
		&\leq R(\pi_{\ep,*})-\frac14\cdot(V_{s,\ep}^*-V_{s,\ep}) + L\cdot\beta_{\ep}.
		\end{aligned}
		\end{equation}
		Here the first line comes from \Cref{lemma:small_relaxation_gain}, the second line comes from the fact that $f(x_1)-f(x_2)\geq\min_x(f'(x))(x_1-x_2)$ for $x_1\geq x_2$ and $f'(x)>0$, and the third line comes from the fact that $\frac{\partial R(\pi)}{\partial V_s}\geq\frac14$ for $V_s\in[0.625, 0.728]$ as $0.728>\frac{8+10\ep}{11+10\ep}$ for $\ep\leq10^{-10}$. Also, since $\pi$ satisfies a low-regret condition, we have
		\begin{equation*}
		R(\pi_{\ep, *})-R(\pi)\leq C_1\cdot T^{-\frac12+\eta}.
		\end{equation*}
		Combining with \Cref{equ:performance_difference_pi_pispstar}, we have
		\begin{equation}
		\label{equ:vs_difference_bound}
		\begin{aligned}
		\frac14(V_{s,\ep}^*-V_{s,\ep})-L\cdot\beta_{\ep}&\leq{C_1} T^{-\frac12+\eta}\\
		\Rightarrow\qquad(V_{s,\ep}^*-V_{s,\ep})&\leq{C_1} T^{-\frac12+\eta}+L\cdot\beta_{\ep}\\
		&\leq(C_1+LC_2) T^{-\frac12+\eta}.
		\end{aligned}
		\end{equation}
		Notice that this is suitable for any $\ep\in[0,10^-10]$, we may have the same result for both $\ep = {C} T^{-\frac12+\eta}$ and for $\ep = 0$. We denote $\ep_0 = 0$ and $\ep_1 = {C} T^{-\frac12+\eta}$ where $C=10^{-10}$. Therefore, we have:
		\begin{equation}
		\label{equ:vs_difference_bound_in_two_problem}
		\begin{aligned}
		(V_{s,\ep_0}^*-V_{s,\ep_0})\leq&(C_1+LC_2) T^{-\frac12+\eta},\\
		(V_{s,\ep_1}^*-V_{s,\ep_1})\leq&(C_1+LC_2) T^{-\frac12+\eta}.
		\end{aligned}
		\end{equation}
		Now let us bound $(\alpha^*_{\ep}-\alpha_{\ep})$ for both $\ep_0$ and $\ep_1$. From \Cref{equ:alpha_bound_one} and \Cref{equ:alpha_bound_three}, we have
		\begin{equation}
		\begin{aligned}
		B_3|_{V_s=V_{s,\ep}}&\leq\alpha_{\ep}\leq B_1|_{V_s=V_{s,\ep}}\\
		B_3|_{V_s=V_{s,\ep}^*}&=\alpha_{\ep}^*= B_1|_{V_s=V_{s,\ep}^*}\\
		\Rightarrow\qquad\min_{V_s\in[0.7, 0.75]}\{\frac{\partial B_3}{\partial V_s}, \frac{\partial B_1}{\partial V_s}\}\cdot(V_{s,\ep}^*-V_{s,\ep})&\leq(\alpha_{\ep}^*-\alpha_{\ep})\leq\max_{V_S\in[0.7, 0.75]}\{\frac{\partial B_3}{\partial V_s}, \frac{\partial B_1}{\partial V_s}\}(V_{s,\ep}^*-V_{s,\ep})\\
		\Rightarrow\qquad0\leq0.05(V_{s,\ep}^*-V_{s,\ep})&\leq(\alpha_{\ep}^*-\alpha_{\ep})\leq0.6(V_{s,\ep}^*-V_{s,\ep}).
		\end{aligned}
		\end{equation}
		Therefore, we have:
		\begin{equation*}
		0\leq V_{r,\ep}^*-V_{r,\ep}=V_{s, \ep}^* + \alpha_{\ep}^*-(V_{s,\ep}+\alpha_{\ep})\leq(1+0.6)(V_{s,\ep}^*-V_{s,\ep})\leq \frac85\cdot(C_1+LC_2) T^{-\frac12+\eta}.
		\end{equation*}
		Therefore, we know that
		\begin{equation}
		\label{equ:ve}
		\begin{aligned}
		V_{r,\ep_0}^*\geq& V_{r,\ep_0}\geq V_{r,\ep_0}^*-\frac85\cdot(C_1+LC_2) T^{-\frac12+\eta},\\
		V_{r,\ep_1}^*\geq& V_{r, \ep_1}\geq V_{r, \ep_1}^*-\frac85\cdot(C_1+LC_2) T^{-\frac12+\eta}.\\
		\Rightarrow\qquad & |V_{r, \ep_1} - V_{r, \ep_0}|\geq |V_{r,\ep_0}^*- V_{r, \ep_1}^*|-(C_1+LC_2) T^{-\frac12+\eta}.
		\end{aligned}
		\end{equation}
		HOWEVER, we have $V_{r, \ep_1} = V_{r, \ep_0}$ since they are the expected proposed price of the same pricing policy $\pi$ in $P_0$ and $P_{\ep}$ where the prices sets are all the same! Therefore, we have $|V_{r,\ep_0}^*- V_{r, \ep_1}^*|-(C_1+LC_2) T^{-\frac12+\eta}\leq0$. Since $V_{r,\ep_0}^*=\frac{43}{58}$ and $V_{r, \ep_1} = \frac{43+10\ep_1}{58-20\ep_1}$, we have $|V_{r, \ep_0}^*-V_{r, \ep_1}^*|=\frac{360\ep}{29(29-10\ep)}\geq\frac{C}3\cdot T^{-\frac12+\eta}$. Since $C_1 =\frac{C}{W}\leq 10^{-11}$ and $C_2=\frac{C}{W\cdot L}\leq\frac1L\times 10^-11$, we know that $|V_{r, \ep_0}^*-V_{r, \ep_1}^*|>(C_1+LC_2) T^{-\frac12+\eta}$, which contradicts to the inequality we derived. Therefore, the lemma is proved by contradiction.
	\end{proof}
	In the following, we set $\ep=\ep_1= {C} T^{-\frac12+\eta}$ where $C=10^-10$ as is defined in the proof of \Cref{lemma:no_perfect_policy}. Now let us go back to the main stream of proving \Cref{theorem:unfairness_lower_bound}. We also make it by contradiction. For any given $C_x$, without loss of generality, we may assume that $C_u\leq C_x$ to be specified later. Define $x=\frac{C_x}{\log T}$, and therefore $C_xT^{\frac12} = T^{\frac12-x}$. We will make use of \Cref{example:random_policy} and \Cref{example:lowerbound_construction}, and let $\eta>x$ to be specified later. Therefore, we have $C_u\cdot T^{\frac12}\leq C_x\cdot T^{\frac12}=T^{\frac12-x}$, which means that the contradiction is a \textbf{sufficient condition} to the following result: Suppose there exists an $x>0$ and an algorithm such that it can always achieve $O(T^{\frac12+x})$ regret with zero procedural unfairness and $O(T^{\frac12-x})$ substantive unfairness. According to \Cref{corollary:3-parity}, we know that any policy $\pi\in\Pi$ are in exact one of those three spaces. In our problem setting, denote the policy we take at time $t=1,2,\ldots, T$ as $\pi_t$.  Now we show that: among all policies $\{\pi_t\}_{t=1}^T$ we have taken, there are at most $O(T^{1-\eta+x})$ policies in all $T$ policies having been played belonging to the Outer Space defined in \Cref{corollary:3-parity}. In fact, for any policy $\pi$ in the Outer Space, we have:
	\begin{enumerate}[label=(\roman*)]
		\item When $\pi\in\bar{A}\bar{B}$, the policy $\pi$ will definitely suffer a regret $C_0\cdot T^{-\frac12+\eta}$, no matter which the problem setting is (i.e., $P_0$ or $P_{\ep}$). In order to guarantee $O(T^{\frac12+x})$ regret, there are at most $N_1=O(T^{1-\eta+x})=o(T)$ rounds to play a policy in $\bar{A}\bar{B}$.
		\item When $\pi\in\bar{C}\bar{D}$, the policy $\pi$ will definitely suffer a substantive unfairness $C_0\cdot T^{-\frac12+\eta}$, no matter which the problem setting is (i.e., $P_0$ or $P_{\ep}$). In order to guarantee $O(T^{\frac12-x})$ regret, there are at most $N_2 = O(T^{1-\eta-x})=o(T)$ rounds to play a policy in $\bar{C}\bar{D}$.
		\item When $\pi\in\bar{A}\bar{D}\vee\bar{B}\bar{C}$, in either $P_0$ or $P_{\ep}$ it suffers something (that could be either $C_0\cdot T^{-\frac12+\eta}$ regret or $C_0\cdot T^{-\frac12+\eta}$ unfairness). As we have to guarantee $O(T^{\frac12+x})$ regret and $O(T^{\frac12-x})$ substantive unfairness, there are still at most $N_3 = O(\max\{T^{1-\eta+x}, T^{1-\eta-x}\})=O(T^{1-\eta+x})=o(T)$.
	\end{enumerate}
	Therefore, the number of rounds when we select and play a policy from Space AC or Space BD is at least $T-o(T)\geq\frac T2$. Notice that if a policy in $AC$, then it performs well in $P_0$ but not necessarily in $P_{\ep}$. Similarly, if a policy in $BD$, then it performs well in $P_{\ep}$ but not necessarily in $P_0$. Therefore, two questions emerges:
	
	\begin{itemize}
		\item How do policies in $AC$ perform in $P_{\ep}$? and How do policies in $BD$ perform in $P_0$? Specifically, we only care about the substantive fairness.
		\item How can we distinguish between $P_{\ep}$ and $P_0$?
	\end{itemize}
	For distinguishablity, denote $F_1(\ep) = \diag\{0.6,0.5-\ep, 0.5-\ep\}$ and $F_2(\ep) = \diag\{0.8,0.8,0.5-\ep\}$. Denote $S_0(\pi):=S(\pi, F_1(0), F_2(0))|_{\ep=0}$ and $S_{\ep}(\pi):=S(\pi, F_1(\ep), F_2(\ep))$. In the following, we propose two lemmas that help us prove. The first lemma, \Cref{lemma:distance_of_disjoint_set}, shows that failing to distinguish would lead to large substantive unfairness, which answers the first question above.
	\begin{lemma}
		\label{lemma:distance_of_disjoint_set}
		There exists a constant $C_{ac}$ such that: for any policy $\pi\in AC$, we have $S_{\ep}(\pi)>C_{ac}\cdot T^{-\frac12+\eta}$. There also exists a constant $C_{bd}$ such that: for any policy $\pi\in BD$, we have $S_0(\pi)>C_{bd}\cdot T^{-\frac12+\eta}$.
	\end{lemma}
	\begin{proof}[Proof of \Cref{lemma:distance_of_disjoint_set} ]
		We firstly prove the first half of this lemma, and then demonstrate the second half (which can be proved in exact the same way.)
		
		First of all, we have the close-form solution to both $P_0$ and $P_{\ep}$ in \Cref{equ:close_form}. Therefore, we have
		\begin{equation}
		\label{equ:plug_pi_0_in_s_e}
		S_{\ep}(\pi_{0, *})=\frac{12\ep(1-2\ep)}{(11-6\ep)(11-10\ep)}.
		\end{equation}
		Now, consider any policy $\pi\in AC$. Similar to the Proof of \Cref{lemma:no_perfect_policy}, we define its accepted prices in $G_1$ and $G_2$ are $V_{s,0}$ and $V_{s,0}+\beta$ where $\beta\in[0,C_2T^{-frac12+\eta}]$. We also denote the expected proposed price in both group as $V_{r,0}=V_{s,0} + \alpha_0$. Also, define a corresponding policy $\check{\pi}$:
		\begin{equation}
		\label{equ:check_pi_0}
		\check{\pi}:=\left\{
		\begin{aligned}
		G_1&:\E[\text{accepted price}]=V_{s,0}, \E[\text{proposed price}]=V_{s,0}+\alpha_{0}\\
		G_2&:\E[\text{accepted price}]=V_{s,0}, \E[\text{proposed price}]=V_{s,0}+\alpha_{0}.
		\end{aligned}
		\right.
		\end{equation}
		Notice that $\pi^1 = (A_1(V_{s,0}, 0))^{-1}[1, V_{r,0}, 0]^{\top}$ and $\pi^2 = (A_2(V_{s,0}, 0))^{-1}[1, V_{r,0}, \beta\cdot\ind^{\top}F_2\pi^2]^{\top}$. In comparison, we have $\check{\pi}^1 = (A_1(V_{s,0}, 0))^{-1}[1, V_{r,0}, 0]^{\top}$ and $\check{\pi}^2 = (A_2(V_{s,0}, 0))^{-1}[1, V_{r,0}, 0]^{\top}$. Therefore, we have:
		\begin{equation}
		\label{equ:policy_lipschitz}
		\begin{aligned}
		\pi^1 & = \check{\pi}^1\\
		\|\pi^2 - \check{\pi}^2\|_1 & = \|(A_2(V_{s,0}, 0))^{-1}([1, V_{r,0}, \beta\cdot\ind^{\top}F_2\pi^2]^{\top}-[1,V_{r,0},0]^{\top})\|_1\\
		&\leq\|(A_2(V_{s,0},0))^{-1}[0,0,\beta]\|_1\\
		&=\|(A_2(V_{s,0},0))^{-1}_{[:,3]}\|_1\cdot\beta\\
		&\leq\frac{100\beta}{9(7V_s-5)}.
		\end{aligned}
		\end{equation}
		Also, since $F_{\min}\leq\ind^{\top}F_2\pi^2\leq1$ and $\|\vv^{\top}F_2\|_1\leq d$ always hold, we know that $\|\frac{\partial S_{\ep}(\pi)}{\partial \pi^2}\|\leq\frac d{F_{\min}}\cdot \|\pi^2\|_1=\frac d{F_{\min}}$. Since $V_{s,0}^*\approx\frac8{11}$ and all $V_{s,0}$ we consider are around it (According to \Cref{equ:vs_difference_bound_in_two_problem}), we may assume that $(7V_s-5)>\frac12\cdot(\frac8{11}-\frac57)>\frac1{200}$ Therefore, we have:
		\begin{equation}
		\label{equ:bound_s_2}
		\begin{aligned}
		|S_{\ep}(\check{\pi})-S_{\ep}(\pi)|&\leq\frac d{F_{\min}}\|\check{\pi}^2-\pi^2\|_2\\
		&\leq \frac {100d \beta}{9(7\cdot V_s - 5)F_{\min}}\\
		&\leq \frac {100d C_2}{9(7\cdot V_s - 5)F_{\min}}\cdot T^{-\frac12+\eta}\\
		&\leq \frac {3000d C_2}{F_{\min}}\cdot T^{-\frac12+\eta}.
		\end{aligned}
		\end{equation}
		Also, according to the proof of \Cref{lemma:no_perfect_policy}, we know that $|V_{s,0}-V_{s,0}^*|\leq(C_1+LC_2)T^{-\frac12+\eta}$ and $|\alpha_0^*-\alpha_0|\leq0.6(C_1+LC_2)T^{-\frac12+\eta}$ (as $\ep = 0$). Plugging in \Cref{equ:pi_1_close_form} and \Cref{equ:pi_2_close_form}, we have:
		\begin{equation}
		\label{equ:lip_check_to_star}
		\begin{aligned}
		\|\pi^1_{0,*}-\check{\pi}^1\|_1\leq & 50\cdot((120+8+10+10\times(8+2))|\alpha_0^*-\alpha_0|+(13+106)|V_{s,0}-V_{s,0}^*|)\\
		\leq & 1309(C_1+LC_2)T^{-\frac12+\eta}\\
		\|\pi^2_{0,*}-\check{\pi}^2\|_1\leq & \frac1{3\times0.2\times3}((4\times36+120+24)|\alpha_0^*-\alpha_0|+(120 + 51+120+39 )|V_{s,0}-V_{s,0}^*|)\\
		\leq & 350(C_1+LC_2)T^{-\frac12+\eta}\\
		\end{aligned}
		\end{equation}
		Therefore, we have:
		\begin{equation}
		\begin{aligned}
		|S_{\ep}(\pi_{0,*})-S_{\ep}(\check{\pi})\leq & \frac d{F_{\min}}\|\pi_{0,*}-\check{\pi}\|_2\\
		=&\frac d{F_{\min}}(\|\pi_{0,*}^1-\check{\pi}^1\|_2+\|\pi_{0,*}^2-\check{\pi}^2\|_2)\\
		\leq&\frac d{F_{\min}}(\|\pi_{0,*}^1-\check{\pi}^1\|_1+\|\pi_{0,*}^2-\check{\pi}^2\|_1)\\
		\leq&\frac d{F_{\min}}\cdot( 1309(C_1+LC_2)T^{-\frac12+\eta} +  350(C_1+LC_2)T^{-\frac12+\eta})\\
		\leq&\frac d{F_{\min}}2000(C_1+LC_2)T^{-\frac12 + \eta}.
		\end{aligned}
		\end{equation}
		Recall that $C_1 = \frac{C}{W}$ and $C_2 = \frac{C}{W\ldots L}$. Now, we let $W= 10^6\frac{d}{F_{\min}}$. Therefore, we have:
		\begin{equation}
		\label{equ:lipschitz_derive_fairness_bound}
		\begin{aligned}
		S_{\ep}(\pi)=&S_{\ep}(\pi)-S_{\ep}(\check{\pi})+S_{\ep}(\check{\pi})-S_{\ep}(\pi_{0,*})+S_{\ep}(\pi_{0,*})\\
		\geq&S_{\ep}(\pi_{0,*})-|S_{\ep}(\pi)-S_{\ep}(\check{\pi})|-|S_{\ep}(\check{\pi})-S_{\ep}(\pi_{0,*})|\\
		\geq&\frac{12\ep(1-2\ep)}{(11-2\ep)(11-6\ep)}- \frac {3000d C_2}{F_{\min}}\cdot T^{-\frac12+\eta}- \frac d{F_{\min}}2000(C_1+LC_2)T^{-\frac12 + \eta}\\
		\geq&\frac1{20}\ep- \frac {5000d (C_1+LC_2)}{F_{\min}}\cdot T^{-\frac12+\eta}\\
		=&\frac1{20}C\cdot T^{-\frac12+\eta} - \frac {5000d C}{F_{\min}\cdot W}\cdot T^{-\frac12+\eta}\\
		\geq&\frac1{20}C\cdot T^{-\frac12+\eta} -\frac1{200}C\cdot T^{-\frac12+\eta}\\
		\geq&\frac1{30}C\cdot T^{-\frac12+\eta}.
		\end{aligned}
		\end{equation}
		Let $C_{ac} = \frac1{30}\cdot C$ and this lemma holds.
	\end{proof}
	Define $\P_{P_0}$ and $\P_{P_{\ep}}$ as the probabilistic distribution of customer's feedback at each round. In order to increase the information for distinguishing between two problem settings, we assume that a customer would always tell us whether or not she accept the price $\$1$, at each time $t=1,2,\ldots, T$. Therefore, both $\P_{P_0}$ and $\P_{P_{\ep}}$ are binomial distributions $B(T, 0.5)$ and $B(T, 0.5-\ep)$. Here we present another lemma, the \Cref{lemma:indistinguishability}, that indicates the hardness of distinguishing the two settings.
	\begin{lemma}
		Consider the $N\geq\frac T 2$ rounds when we play a policy in $AC\vee BD$. For any algorithm $\phi$, denote $\phi_t=1$ if $\pi_t\in AC$ and $\phi_t=0$ if $\pi_t\in BD$. Then we have:
		\begin{equation}
		\label{equ:lemma_i_t_lower_bound}
		\max\{\E_{P_0}[\sum_{t=1}^N\phi_t], \E_{P_{\ep}}[\sum_{t=1}^N(1-\phi_t)]\}\geq\frac1{8} T\cdot \exp(-T^{2\eta}).
		\end{equation}
		\label{lemma:indistinguishability}
	\end{lemma}
	\begin{proof}[Proof of \Cref{lemma:indistinguishability} ]
		In fact, we have:
		\begin{equation}
		\label{equ:lecam_application}
		\begin{aligned}
		\max\{\E_{P_0}[\sum_{t=1}^N\phi_t], \E_{P_{\ep}}[\sum_{t=1}^N(1-\phi_t)]\}&\geq\frac{\E_{P_0}[\sum_{t=1}^N\phi_t] + \E_{P_{\ep}}[\sum_{t=1}^N(1-\phi_t)]}2\\
		&=N\cdot\frac{\P_{P_0}[\phi_t == 1] + \P_{P_{\ep}}[\phi_t == 0]}2\\
		&\geq\frac T 4 \cdot(\P_{P_0}[\phi_t == 1] + \P_{P_{\ep}}[\phi_t == 0])\\
		&\geq\frac T 8 \cdot\exp(-N\cdot KL(\P_{P_0}||\P_{P_{\ep}}))\\
		&\geq\frac T 8 \cdot\exp(-N\cdot KL(Ber(0.5)||Ber(0.5-\ep)))\\
		&\geq\frac T 8 \cdot\exp(-N\cdot 12\ep^2)\\
		&=\frac T 8 \cdot\exp(-N\cdot 12(C\cdot T^{-\frac12 + \eta})^2)\\
		&\geq\frac T 8 \cdot\exp(-12C^2T^{2\eta})\\
		&\geq\frac T 8 \cdot\exp(-T^{2\eta}).
		\end{aligned}
		\end{equation}
		Here the first line is for $\max\geq\text{average}$, the second is by definition of $\phi_t$, the third line is for $N\geq\frac T 2$, the fourth line is from Fano's Inequality that $\P_0[\phi==1] + \P_1[\phi==0]\geq\frac12\cdot\exp\{-N\cdot KL(\P_0||\P_1)\}$ for any distributions $\P_0$ and $\P_1$, the fifth line is by definition of $P_0$ and $P_{\ep}$ that they are only different in the customers' feedback satisfying $Ber(0.5)$ and $\Pr=0.5-\ep$ for some actions, respectively, the sixth line is from \Cref{lemma_kl_ber}, the seventh line is for $\ep=C\cdot T^{-\frac12+\eta}$, the eighth line is for $N\leq T$, and the last line is for $12C^2\leq 1$.
	\end{proof}
	With the two lemma above, we know that
	\begin{itemize}
		\item For any algorithm $\phi$, we either run at least $\frac T 8\cdot\exp(-T^{2\eta})$ rounds with some $\pi_t\in AC$ when the problem setting is $P_{\ep}$, or run at least $\frac T 8 \cdot\exp(-T^{2\eta})$ rounds with some $\pi_t\in BD$ when the problem setting is $P_0$, according to \Cref{lemma:indistinguishability}.
		\item For each round we mismatching the problem setting, we will suffer a $\min\{C_{ac}, C_{bd}\}\cdot T^{-\frac12+\eta}$ unfairness, according to \Cref{lemma:distance_of_disjoint_set}.
	\end{itemize}
	Given these two facts, denote $C_{\min}:=\frac18\min\{C_{ac}, C_{bd}\}$ and we at least have $C_{\min}T\cdot\exp(-T^{2\eta})\cdot T^{-\frac12+\eta}$ unfairness. For $x=\frac{C_x}{\log T}$ with any constant $C_x$, we let $\eta = \frac{3x}{2}=\frac{3C_x}{2\log T}$ and therefore $\eta>x$. As a result, we have
	\begin{equation}
	\label{equ:derive_c_u}
	\begin{aligned}
	C_{\min}T\cdot\exp(-T^{2\eta})\cdot T^{-\frac12+\eta}&=C_{\min}\exp(-T^{2\eta}+\eta\log T)T^{\frac12}\\
	&=C_{\min}\exp(-T^{2\cdot\frac{C_x}{\log T}} + \frac{3C_x}{2\log T}\cdot\log T)T^{\frac12}\\
	&=C_{\min}\exp(-\exp(2\cdot\frac{C_x}{\log T}\cdot\log T) + \frac {3C_x}{2})T^{\frac12}\\
	&=C_{\min}\exp(-\exp(2C_x)+\frac{3C_x}2)T^{\frac12}
	\end{aligned}
	\end{equation}
	Let $C_u = \frac{C_{\min}\exp(-\exp(2C_x)+\frac{3C_x}2)}2$, and then the result of the equation above contradicts with the suppose that the unfairness does not exceed $C_u\cdot T^{\frac12}$. Therefore, we have proved the theorem.

\end{proof}
\section{More Discussion}
\label{appendix:more_discussion}
Here we discuss more on the problem settings we assumed, the techniques we used and the social impacts our algorithm might have, as a complement to \Cref{sec:discussion}.
\subsection{Potential Generalizations of Current Problem Setting.} Currently we make a few technical assumptions that qualify the applications of our algorithm. In fact, these assumptions are mild and can be released by some tricks: On the one hand, we can always meet the requirement of \Cref{assumption:f_min} by reducing $v_d$. By running a binary-search algorithm for the highest acceptable price (with constant acceptance probability), we can find the feasible $v_d$ within $O(\log T\log T)$ rounds (where $\log T$ for binary search and another $\log T$ for the concentration of a constant-expectation random variable, as we did in estimating $F_{\min}$). Since $O(\log T\log T)$ is much smaller than $O(\sqrt{T})$ as the optimal regret and unfairness, this would not harm the regret and unfairness substantially. On the other hand, we assume the prices to be chosen from a fixed and finite price set $\mathbf{V}$, which not only restricted our action but might lead to suboptimality from the perspective of a larger scope. In fact, if we allow the prices to be selected in the whole $[0,1]$ range, a pricing policy can be a tuple of two continuous distributions over $[0,1]$. To solve this problem, we may parametrize the distribution and learn the best parameters. We may also discretize the price space into small grids, i.e. prices are $\mathbf{V}=\{\gamma, 2\gamma, \ldots, (d-1)\gamma, d\gamma=1\}$, where $\gamma=T^{-\alpha}$ with some constant $\alpha$ and d= $T^{\alpha}$ as a consequence. It is intrinsically a specific way of parametrization. According to the ``half-Lipschitz'' nature of pricing problem as well as our \Cref{lemma:small_relaxation_gain} along with the Lipschiness of $S(\pi; F_1, F_2)$, we know that the per-round discretization error would be upper bounded by $O(T^{-\alpha})$. Let the cumulative discretization error $O(T^{\frac1-\alpha})$ balances the cumulative regret (or substantive unfairness), i.e., $O(d^{\frac32}\sqrt{T})=0(T^{\frac12+\frac{3\alpha}2})$, we can achieve an upper bound on both the regret  at $O(T^{\frac45})$ by letting $\alpha=T^{\frac15}$. However, this is not optimal as we only match the upper and lower bounds w.r.t. $T$ but not to $d$. Therefore, it would also be an interesting problem to see the minimax regret/unfairness dependent w.r.t. $d$.

Besides of the assumptions we have made, there are other notions regarding our problem setup that can be generalized. Firstly, we may generalize our problem setting from two groups to multiple $G\geq3$ groups. Again, the feasible set is not empty as we can always propose the same fixed price to all groups. However, there is not a directly generalization of the fairness definition, which we defined as the difference of the expectation of certain amount between two groups. We might defined it as ``pairwise unfairness'' by comparing the same difference among each pair of groups and adding them up, but this is not rational: Consider the case when the expected proposed/accepted prices in $(G-1)$ groups are very high and the last one is very low, and compare this case with another case when the expected proposed/accepted prices in $1$ groups are very high and the other $(G-1)$ groups has a very low expected prices. The unfairness in these two cases should be definitely different, as the first seems more acceptable (i.e., being kind to only the minority versus being kind to only the majority). However, their ``pairwise unfairness'' are exactly the same in these two cases. Therefore, a better notion of procedural/substantive unfairness should be established for multiple $G\geq3$ groups.

Secondly, we may generalize the modeling on customers from i.i.d. to strategic. For example, what if a customer tries multiple times until getting the lowest price of the distribution for this group. This is an adaptive adversary and therefore very hard to deal with even in the simplest decision-making process such as bandits. 

Thirdly, we may also include more fairness concern. Currently we are considering the two types of fairness, but we define the cumulative fairness based as the summation of expected per-round unfairness. This definition does not take into consideration the changing of policies. For example, if we propose a fair policy at each round, but the policies over time changes drastically, then it is hard for the customers to feel or experience such a fairness. In our algorithm design of FPA, we always play the same policy for at least $\frac{\tau_k}{2d}=\Omega(\sqrt{T})$ rounds as a batch until the policy changes. This is a long enough time period for customers to experience fairness since at least a $Omega(\sqrt{T})$ number of customers from both groups would come and buy items under the same policy according to the Law of Large Numbers. However, this would still cause a feeling of unfairness for the two customers who are arriving almost simultaneously but the policy is just changed after the first customer buy or decide not to buy. Therefore, there exists necessity for us to consider the time/individual fairness under this online pricing problem scheme.

\subsection{Potential Generalization of Techniques}
Here we discuss a little bit more on the probable extension of the techniques we developed in our algorithm design and analysis.

{\paragraph{From Two Groups to Multi Groups}
	Our problem setting assumes that there are two groups of customers in total. We choose to study a two-group setting to simplify the presentation. In practice, however, it is very common that customers are coming from a number of groups with different valuations even on the same product. In fact, we believe it straightforward to extend our methodologies and results to multi-group settings, as long as we determine a metric of multi-group unfairness. For instance, if we choose to define the multi-group unfairness as the summation of pairwise unfairness of all pairs of groups, we may adjust our algorithm by lengthening each epoch by $G/2$ times and keeping everything the same as in this paper. In this way, the upper regret bound would be $\tilde{O}(G^2\sqrt{T}d^{2/3})$, which is $O(G^2)$ times as we have shown in this paper. Therefore, it is still optimal w.r.t. $T$ up to iterative-log factors.
}

\paragraph{A Good-and-exploratory Policy Set}
Our algorithm FPA maintains and updates a ``good-and-exploratory'' policy $A_k$ in each epoch. Each policy in this set performs close to the optimal policy in both regret and unfairness reductions. A similar idea in reinforcement learning related research exists in \citet{qiao2022sample} where they select policies that visit each (horizon, state, action) tuple most sufficiently while ensuring that the policy is low in regret. In fact, if we imagine an ``exploratory'' policy as the one that would elevate the most ``information'' (i.e., that would reduce the most uncertainty), then the ``good-and-exploratory'' policy-selection process is equivalent to an ``Upper Confidence Bound'' method \citet{lai1985asymptotically} where we always pull the arm with the highest upper confidence bound in a multi-armed bandit. The only difference is that: for traditional exploration-and-exploitation balancing algorithm, we only need to improve our estimation on the parameters of these optimal or near-optimal policies. However, in our problem setting, we have to guarantee a uniform error bound, i.e., we have to improve our estimation on all parameters instead of only those optimal-related ones. This is because that we have to improve the estimation on constraints as well as on the revenue function. In our algorithm design, we handle this problem by keeping eliminating a feasible policy set, which in turn releases the algorithm from estimating those unnecessary parameters. In a nutshell, our methods can be applied broadly in online-decision-making problems.
\paragraph{Unfairness Lower Bound Proof on Optimal Algorithms}
The main idea of our proof of the $\Omega(\sqrt{T})$ unfairness lower bound on any algorithm with $O(\sqrt{T})$ optimal regret is to construct a trade-off on unfairness and regret between two adjacent problem settings. We first bound the ``bad policies'' away from each problem setting, to avoid those policies that are super fair in both setting but performs poor in both setting as well. Then we show that policies with small-enough regret and unfairness on one setting should suffer a large regret on the other. Finally we end the proof by showing that we will definitely make $\tilde{\Theta}(T)$ times of mistakes in expectation, according to information theory. We believe that this scheme can be used in proving a variety of trading-off lower bounds.
\subsection{Social Impacts}
In this work, we develop methods to prompt the procedural and substantive fairness of customers from all groups. We believe that our techniques and results would enhance the unity of people with different gender, race, age, cultural backgrounds, and so on. However, it is definitely correct that we have to \emph{treat differently} to different group of people. In order to ensure the fairness from customers' perspective, the seller is required to behave unfairly. Of course we could partly get rid of this issue by leaving the generating process of a random price to the nature, i.e., we let each customer draw a coupon from a box randomly. However, this only means that the seller's pricing process is fair but not leads to a fair result, as customers' coupon varies a lot from person to person. This turn out to be the exact issue named as ``pricing and price fairness'' proposed in \citet{chapuis2012price} regarding the fairness of a seller's behavior. Maybe in the future we could develop an algorithm that is not only profitable but also ensures the fairness from both the seller and the customers' perspective, which could be a truly ``doubly fair'' dynamic pricing. 

\end{appendices}
\bibliographystyle{abbrvnat}  

\bibliography{ref_log}

\end{document}